\documentclass[letterpaper]{article}
\usepackage{uai2020}
\usepackage[margin=1in]{geometry}

\usepackage{times}

\title{Online Parameter-Free Learning of Multiple Low Variance Tasks}

\author{~~~ Giulia Denevi$^{1}$ ~~~~~~~~~~~~~~~~~ Dimitris Stamos$^{2}$ ~~~~~~~~~~~~~~~~~ Massimiliano Pontil $^{1,2}$ ~~~ \\ 
{\small \hspace*{-2.0em} ~~~ giulia.denevi@iit.it ~~~~~~~~~~~~~~~~~ d.stamos.12@ucl.ac.uk ~~~~~~~~~~~~~~~~~ massimiliano.pontil@iit.it} \vspace{.1cm} \\
\small{$^1$ Computational Statistics and Machine Learning, Istituto Italiano di Tecnologia, 16163 Genova, Italy} \\
\small{$^2$ Computer Science Department, University College of London, WC1E 6BT, London, United Kingdom}}

\usepackage{amsmath,amssymb,amsthm} 
\usepackage[nameinlink,capitalize]{cleveref}
\usepackage[english]{babel}
\usepackage{algorithm}
\usepackage{algpseudocode}
\usepackage{appendix}
\usepackage{thmtools,thm-restate}
\usepackage{etoolbox}
\usepackage{graphicx}
\usepackage{subfigure}
\usepackage{todonotes}

\newcommand{\X}{\mathcal{X}}
\newcommand{\Y}{\mathcal{Y}}

\newcommand{\E}{\mathcal{E}}
\newcommand{\U}{\mathcal{U}}
\newcommand{\EE}{\mathbb{E}}

\newcommand{\B}{\mathcal{B}}

\newcommand{\proj}{\ensuremath{\text{\rm proj}}}

\newcommand{\argmin}{\operatornamewithlimits{argmin}}

\newcommand{\wmu}{{w_\mu}}

\newcommand{\ee}{{\mathcal{E}}}
\newcommand{\rx}{{R}}

\newcommand{\task}{\mu}
\newcommand{\env}{\rho}
\newcommand{\Zn}{Z}

\newcommand{\data}{Z}

\newcommand{\cR}{\mathcal{R}}
\newcommand{\Real}{\mathbb{R}}
\newcommand{\Exp}{\mathbb{E}}

\newcommand{\vertiii}[1]{{\left\vert\kern-0.25ex\left\vert\kern-0.25ex\left\vert #1 
    \right\vert\kern-0.25ex\right\vert\kern-0.25ex\right\vert}}

\crefname{assumption}{Asm.}{Asm.}
\crefname{equation}{}{}
\crefname{equation}{Eq.}{Eqs.}
\crefname{figure}{Fig.}{Fig.}
\crefname{table}{Tab.}{Tab.}
\crefname{section}{Sec.}{Sec.}
\crefname{theorem}{Thm.}{Thm.}
\crefname{proposition}{Prop.}{Prop.}
\crefname{fact}{Fact}{Facts}
\crefname{lemma}{Lemma}{Lemmas}
\crefname{corollary}{Cor.}{Cor.}
\crefname{example}{Ex.}{Ex.}
\crefname{remark}{Rem.}{Rem.}
\crefname{algorithm}{Alg.}{Algorithms}
\crefname{appendix}{App.}{App.}
\crefname{algorithm}{Alg.}{Alg.}

\declaretheorem[name=Theorem,refname=Thm.]{theorem}

\declaretheorem[name=Proposition,refname=Prop.,sibling=theorem]{proposition}

\declaretheorem[name=Corollary,refname=Cor.,sibling=theorem]{corollary}

\declaretheorem[name=Assumption,refname=Asm.]{assumption}

\makeatletter
\renewenvironment{proof}[1][\proofname]{\par
  \pushQED{\qed}
  \normalfont \topsep6\p@\@plus6\p@\relax
  \trivlist
  \item[\hskip\labelsep
        \bfseries
    #1\@addpunct{.}]\ignorespaces
}{
  \popQED\endtrivlist\@endpefalse
}
\makeatother
\def\eop{$\rule{1.3ex}{1.3ex}$}
\renewcommand\qedsymbol\eop

\makeatletter
\newenvironment{proofsketch}[1][{\bf Proof Sketch.}]{\par
  \pushQED{\qed}
  \normalfont \topsep6\p@\@plus6\p@\relax
  \trivlist
  \item[\hskip\labelsep
        \bfseries
    #1\@addpunct{}]\ignorespaces
}{
  \popQED\endtrivlist\@endpefalse
}
\makeatother
\def\eop{$\rule{1.3ex}{1.3ex}$}
\renewcommand\qedsymbol\eop

\makeatletter
\newenvironment{proofGD}[1][\proofname]{\par
  \pushQED{\qed}%
  \normalfont \topsep6\p@\@plus6\p@\relax
  \trivlist
  \item[\hskip\labelsep
        \bfseries
    #1%\@addpunct{.}
    ]
\ignorespaces
}{%
  \popQED\endtrivlist\@endpefalse
}
\makeatother
\def\eop{$\rule{1.3ex}{1.3ex}$}
\renewcommand\qedsymbol\eop

\begin{document}

\maketitle

\begin{abstract}
We propose a method to learn a common bias vector for a growing sequence of low-variance tasks. Unlike state-of-the-art approaches, our method does not require tuning any hyper-parameter. Our approach is presented in the non-statistical setting and can be of two variants. The ``aggressive'' one updates the bias after each datapoint, the ``lazy'' 
one updates the bias only at the end of each task. We derive an across-tasks regret bound for the method. When compared to state-of-the-art approaches, the aggressive variant returns faster rates, the lazy one recovers standard rates, but with no need of tuning hyper-parameters. We then adapt the methods to the statistical setting: the aggressive variant becomes a multi-task learning method, the lazy one a meta-learning method. 
Experiments confirm the effectiveness of our methods in practice.
\end{abstract}

%--------------------------------------------------------------------------------------------------------

\section{INTRODUCTION}
\label{introduction}

A long standing problem in machine learning is to develop algorithms that can learn effectively on the basis of only few training examples. To this end, a basic principle that  has proven fruitful is to leverage similarities among a set of tasks in order to facilitate their learning process by their corresponding training samples. This basic principle has been studied both from a multi-task learning (MTL) and a meta-learning or learning-to-learn (LTL) perspective. In the first case we wish to perform well on the same tasks used during training, in the second case we aim to extract ``knowledge'' from the observed tasks that would be useful for solving \emph{new} (possible yet unseen) similar tasks. We refer to \cite{baxter2000model,caruana1998multitask,maurer2005algorithmic,maurer2016benefit,thrun2012learning} and references therein for a detailed discussion on these frameworks.

Both multi-task learning and meta-learning were originally investigated in 
the setting in which the tasks' data are assumed to be independently sampled 
from an underlying probability distribution and they are processed in one entire batch,
 see for instance ~\cite{baxter2000model,evgeniou2005learning,maurer2013sparse,maurer2016benefit,pentina2014pac}. Quite recently, significant progress has been made towards the design
of more efficient algorithms in which the data are sequentially processed and may even be adversarially generated, see \cite{alquier2016regret,balcan2019provable,cavallanti2010linear,denevi2019learning,denevi2018incremental, denevi2018learning,finn2019online,pentina2016lifelong}. 
In this work, we focus on the 
so-called Online-Within-Online (OWO) setting, in which both the tasks and their 
samples are observed sequentially.

Nevertheless, the existing multi-task learning or meta-learning methods in the literature require tuning hyper-parameters and their proper choice is necessary in order to demonstrate the advantage of such methods over the baseline algorithm learning  
the tasks independently. Typically, in practice, this bottleneck is addressed either by an expensive validation procedure in the statistical setting, or by the so-called doubling trick procedures in the adversarial setting.
In this work, we wish to design OWO parameter-free methods that are well suited to 
address an increasing sequence of low-variance tasks. To this end we consider a within-task variant of the online parameter-free algorithm by \cite{cutkosky2018black}, in which the iterates are translated by a common bias vector. The main goal of this work 
is to design and analyze a parameter-free procedure to learn a good bias directly from a 
sequence of observed tasks. 

{\bf Contributions.} 
We first show that, similarly to what already observed for other
families of algorithms requiring tuning of hyper-parameters
\cite{balcan2019provable,denevi2019learning,denevi2019online},
also for the parameter-free family considered here, setting the ``right' 
bias can be advantageous with respect to (w.r.t.) learning the tasks 
independently by the unbiased algorithm, when the variance of the target 
tasks' vectors is sufficiently small.
After this, the main contribution of this work is to develop a 
parameter-free method aiming at inferring a good bias from a
sequence of tasks in the OWO framework. 
The method is originally presented in a non-statistical setting
and it is able to incrementally process a growing sequence of 
tasks. Our method can be of two variants: an 
``aggressive'' one in which the bias 
vector is updated after each point, and a ``lazy'' one, in which the 
bias' update is performed only at the end of each task training sequence. 
We then derive an across-tasks regret bound for the proposed 
method. In the aggressive case the bound enjoys faster rates w.r.t. the 
state-of-the-art approaches for growing tasks' sequences, while, in the 
lazy case we recover standard rates, but with no need of hyper-parameters' tuning.
Next, we show that both methods and the corresponding bounds can be 
adapted to the statistical setting. Specifically, the aggressive variant 
can be converted into a multi-task learning method, whereas the lazy 
variant can be translated into a meta-learning method, 
generalizing also to new tasks. 
Finally, we test the performance of our methods in numerical experiments. 

{\bf Paper Organization.} We start from describing our setting and recalling some basics on parameter-free online learning that will be employed throughout this work in \cref{setting} and \cref{preliminaries}, respectively. In \cref{within_task_family} we introduce the biased family of within-task algorithms our method is based on. 
After this, in \cref{motivation_bias}, we justify our choice, characterizing the settings in which an appropriate choice of the bias can bring advantages over learning the tasks independently. In \cref{estimation_bias} we describe the aggressive variant of our method and we show that it is able to infer a ``good'' bias vector from a sequence of tasks' datasets providing comparable guarantees to the best bias vector in hindsight. In \cref{statistical_setting} we describe how the method can be converted into a multi-task method in the statistical setting. The description and the analysis of the lazy variant of the method are postponed to \cref{lazy_version}.  
Finally, in \cref{experiments} we test our method in practice and in \cref{conclusion} we draw our conclusion. The proofs we skipped in the main body are postponed to the appendix.

\noindent {\bf Previous Work.} 
The idea of inferring a common bias vector shared among a 
set of low-variance tasks is a well-established and simple approach.
It was originally investigated in the multi-task learning setting
for a finite set of tasks \cite{cavallanti2010linear,evgeniou2005learning,maurer2006}.
The success of this approach in this setting motivated its application 
also to meta-learning, both in batch and online \cite{balcan2019provable,denevi2019learning,denevi2018learning,denevi2019online,khodak2019adaptive,pentina2014pac} 
fashion. The problem of inferring a good bias shared among a set of tasks 
is also closely related to the fine tuning problem (see e.g. \cite{finn2019online}), 
where the goal is to find a good starting point for a specific family of learning algorithms 
over a set of tasks. All the works mentioned above are innovative in their own aspects,
however, they require tuning at least one hyper-parameter. Among them, \cite{khodak2019adaptive} is perhaps the most careful in this aspect, since it develops methods in which the hyper-parameters are adaptively chosen, but in order to 
reach this target, the authors require to constrain the weight vectors to a bounded set.
This does not solve completely the issue above, since in practice one has still to 
choose an appropriate set. The critical aspect of designing parameter-free online algorithms has been already pointed out and addressed in the single task setting, 
see e.g. \cite{orabona2014simultaneous,orabona2016coin,streeter2012no}
and references therein. In this work we show how ideas developed in those papers 
for the single-task setting can be applied to design online multi-task learning 
and meta-learning methods that are well suited to low-variance sequences of tasks
and do not require tuning any hyper-parameter.

%---------------------------------------------------------------------------------------------------------

\section{SETTING}
\label{setting}

In this work, we consider the OWO setting outlined 
in \cite{balcan2019provable,denevi2019online,khodak2019adaptive},
in which, the learner is asked to tackle a sequence of online supervised tasks. 

Each task is associated to an input space $\X$ and an output space $\Y$.
The learner incrementally receives a sequence of datapoints $\data = 
(z_i)_{i = 1}^n = ( x_{i}, y_{i})_{i=1}^n \in ( \X \times \Y )^n$ from
the task and is asked to make a prediction after each point is observed. 
Specifically, at each step $i \in \{1, \dots, n\}$: 
(a) a datapoint $z_i = (x_i, y_i)$ is observed, 
(b) the learner incurs the error $\ell_i(\hat y_i)$, where $\ell_i(\cdot) = \ell(\cdot,y_i)$ for a loss function $\ell$ and $\hat y_i$ is the current outcome (prediction) of the
algorithm, (c) the algorithm updates its prediction $\hat{y}_{i+1}$ using the 
last point it has received. Throughout we let $\X \subseteq \Real^d$, $\Y \subseteq \Real$ and we consider algorithms that perform linear predictions of the form $\hat{y}_{i} = \langle x_i, w_i \rangle$, where $(w_i)_{i = 1}^n$ is a sequence of weight vectors updated by the algorithm and $\langle \cdot, \cdot \rangle$ denotes the standard inner product in $\Real^d$. This assumption can be relaxed by introducing a feature map 
on the inputs. The performance of the algorithm is evaluated by looking at the regret of its iterates over the dataset $\Zn$, i.e.
\begin{equation}
\sum_{i=1}^n \ell_i(\langle x_i, w_i \rangle)
- \min_{w \in \Real^d} \sum_{i=1}^n \ell_i(\langle x_i, w \rangle).
\end{equation}
%where, we are assuming the existence of the minimum norm vector
%$\hat w$ minimizing $\sum_{i=1}^n \ell_i(\langle x_i, \cdot \rangle)$
%over $\Real^d$.

The algorithm we will use in our framework is identified by a bias (meta-parameter) $\theta \in \Real^d$ and the aim is to adapt $\theta$ to a sequence of learning tasks.
To this end, we introduce one more algorithm (a meta-algorithm) that updates the bias as the tasks are incrementally observed. 
We consider two variants of such an algorithm. The first one updates 
the bias after each point is observed and the second variant updates the bias only at the end of each task's training sequence. As we shall see, the main advantage of the first strategy will be to obtain faster learning bounds. However, when we move to the 
statistical setting, the first variant can be converted into a multi-task learning method, while, 
the second into a meta-learning method, able to generalize also across the tasks.

More precisely, denoting by $T$ the number of tasks, for each task $t \in \{1,\dots,T\}$, we let $\data_t = ( x_{t,i}, y_{t,i})_{i=1}^n$ be the corresponding data sequence. Throughout this work, we follow the convention adopted in \cite{denevi2019online} and we use the double subscript notation {``$_{t,i}$'', to denote the $\{$outer, inner$\}$} task index. While, we use $k = k(t,i) = (t-1)n + i \in \{1, \dots, Tn \}$ to denote the index counting the global number of datapoints received by the algorithm. At each time $k = k(t,i)$: (a) the algorithm receives the point $z_{t,i} = (x_{t,i},y_{t,i})$, 
(b) the algorithm incurs the error $\ell_{t,i}(\langle x_{t,i}, w_{t,i} \rangle)$, where 
$\ell_{t,i}(\cdot) = \ell (\cdot, y_{t,i})$ and $w_{t,i}$ is the current within-task iteration,
(c) the bias (and consequently, the inner algorithm) is updated in 
$\theta_{k+1}$ for the aggressive variant or it is kept frozen to $\theta_t$ 
for the lazy variant until the entire task's dataset has been observed, 
(d) the algorithm performs one updating step by the inner algorithm with the 
current meta-parameter, returning the predictor vector $w_{t,i+1}$.
In a very natural way, the performance of the entire procedure above is measured 
by the regret accumulated across the tasks, i.e. 
\begin{equation*}
\sum_{t = 1}^T \Bigg ( \sum_{i=1}^n \ell_{t,i}(\langle x_{t,i}, w_{t,i} \rangle )
- \min_{w_t \in \Real^d} \sum_{i=1}^n \ell_{t,i}(\langle x_{t,i}, w_t \rangle ) \Bigg ).
\end{equation*}
%where, for each task, we are again assuming the existence of the minimum norm
%vector $\hat w_t$ minimizing $\sum_{i=1}^n \ell_{t,i}(\langle x_{t,i}, \cdot \rangle)$
%over $\Real^d$.

We conclude this section by introducing the following standard assumption
which will be used in the following.

\begin{assumption}[Bounded Inputs and Convex Lipschitz Loss]
\label{ass1}
Let $\ell(\cdot, y)$ be convex and $L$-Lipschitz for any $y \in \Y$
and let $\X \subseteq \B(0, \rx)$, where, for any center $c \in \Real^d$ 
and radius $r > 0$, we have introduced the Euclidean ball
\begin{equation}
\B(c,r) = \Big \{ v \in \Real^d: \| v - c \| \le r \Big \}.
\end{equation}
\end{assumption}

%-------------------------------------------------------------------------------------------------------

\section{PRELIMINARIES}
\label{preliminaries}

\begin{algorithm}[t]
\caption{{\fontsize{9pt}{10pt}\selectfont One-Dimension Coin Betting Algorithm
based on Krichevsky-Trofimov (KT) estimator, \cite[Alg. $1$]{orabona2016coin}}} \label{coin_betting_alg}
{\fontsize{9pt}{10pt}\selectfont
\begin{algorithmic}
\State
\State {\bfseries Input} $(g_k)_{k = 1}^K$, $g_k \in \Real$,
$| g_k | \le C$, $\epsilon > 0$
\vspace{.1cm}
\State {\bfseries Initialize} $b_1 = 0$, $u_1 = \epsilon$, $p_1 = b_1 u_1$
\vspace{.1cm}
\State {\bfseries For} $k = 1, \dots, K$
\vspace{.1cm}
\State ~ Receive $g_k$
\vspace{.1cm}
\State ~ Define $u_{k+1} 
%= \epsilon - \frac{1}{C} \sum_{j = 1}^k g_j p_j
= u_k - \frac{1}{C} g_k p_k$ 
\vspace{.1cm}
\State ~ Define $b_{k+1} 
%= - \frac{1}{C k} \sum_{j = 1}^k g_j
= \frac{1}{k} \big ( (k-1) b_k - \frac{1}{C} g_k \big )$ 
\vspace{.1cm}
\State ~ Update $p_{k+1} = b_{k+1} u_{k+1}$
\vspace{.1cm}
\State {\bfseries End}
\vspace{.1cm}
\State {\bfseries Return} $(p_k)_{k = 1}^K$
\end{algorithmic}}
\end{algorithm}

Our method is based on parameter-free online learning.
In this section, we briefly recall two well-known 
parameter-free online algorithms: the one-dimension 
coin betting algorithm in \cref{coin_betting_alg}
and the online projected subgradient algorithm in \cref{proj_sub_alg}. 
In the following, we will use these algorithms to build our framework. 
We note that \cref{proj_sub_alg} does not require tuning any 
hyper-parameter and \cref{coin_betting_alg} requires choosing just 
one hyper-parameter (the initial wealth $\epsilon > 0$). However,
as we will see in the following, there is a quite wide range in which the 
choice of such a hyper-parameter does not affect the overall performance of
the algorithm. For this reason, both \cref{coin_betting_alg} and 
\cref{proj_sub_alg} can be considered parameter-free algorithms.
We start from describing \cref{coin_betting_alg}. 

{\bf One-Dimension Coin Betting Algorithm.}
\cref{coin_betting_alg} coincides with the scalar version 
of the Krichevsky-Trofimov (KT) algorithm described in \cite[Alg. $1$]{orabona2016coin}.
The algorithm takes in input an initial wealth $\epsilon > 0$. At each 
iteration $k$, the algorithm receives a value $g_k \in \Real$ with absolute 
value $| g_k | \le C$ for some $C > 0$, it updates 
a betting fraction $b \in \Real$ and a
wealth $u \in \Real$ and, then, it multiplies them together
to update the global iteration $p = b u$. The linear regret 
of \cref{coin_betting_alg} can be bounded as described in 
the following proposition.

\begin{proposition}
[Regret Bound for \cref{coin_betting_alg}, {\cite[Cor. $5$]{orabona2016coin}}]
\label{coin_betting_alg_regret}
The iterations $(p_k)_{k = 1}^K$ returned by \cref{coin_betting_alg} satisfy the 
following linear regret bound w.r.t. a competitor scalar $p \in \Real$ 
%\gd{positive?}
\begin{equation}
\sum_{k = 1}^K g_k \bigl( p_k - p \bigr) \le 
C \Bigg [ \epsilon + \Phi \Big( \epsilon^{-1} | p| K \Big) | p |\sqrt{K} \Bigg ]
\end{equation}
where, for any $a \in \Real$, we have introduced the
function $\Phi (a) = \sqrt{\log \big ( 1 + 24 a^2 \big )}$.
\end{proposition}
As we can see from the bound above, the dependency of the
bound on the hyper-parameter $\epsilon$ is not problematic;
any value in $[1,\sqrt{K}]$ does not affect the $\sqrt{K}$ rate.
We now recall the main properties of \cref{proj_sub_alg}.

\begin{algorithm}[t]
\caption{{\fontsize{9pt}{10pt}\selectfont Online Projected Subgradient Algorithm, \cite[Alg. $6$]{hazan2016introduction}}} \label{proj_sub_alg}
{\fontsize{9pt}{10pt}\selectfont
\begin{algorithmic}
\State
\State {\bfseries Input} $\B \subset \Real^d$,
$(g_k)_{k = 1}^K$, $g_k \in \Real^d$, $\| g_k \| \le C$
\vspace{.1cm}
\State {\bfseries Initialize} $v_1 \in \B$
\vspace{.1cm}
\State {\bfseries For} $k = 1, \dots, K$
\vspace{.1cm}
\State ~ Receive $g_k$
\vspace{.1cm}
\State ~ Define $\gamma_k = \frac{{\rm diam}(\B)}{C \sqrt{2 k}}$
\vspace{.1cm}
\State ~ Update $v_{k+1} 
= \proj_\B \bigl( v_k - \gamma_k g_k \bigr)$ 
\vspace{.1cm}
\State {\bfseries End}
\vspace{.1cm}
\State {\bfseries Return} $(v_k)_{k = 1}^K$
\end{algorithmic}}
\end{algorithm}

{\bf Projected Online Subgradient Algorithm.}
\cref{proj_sub_alg} coincides with \cite[Alg. $6$]{hazan2016introduction}.
The algorithm takes in input a convex, closed and non-empty set 
$\B \subset \Real^d$ with diameter
\begin{equation}
{\rm diam}(\B) = \sup_{v, v' \in \B} \| v - v' \|.
\end{equation}
At each iteration $k$, the algorithm receives a vector $g_k \in \Real^d$ with norm 
$\| g_k \| \le C$ for some $C > 0$, it performs a descent step along this vector with 
an appropriate length and, then, it projects the resulting vector on the set $\B$.
The linear regret of \cref{proj_sub_alg} can be bounded 
as described in the following proposition.

\begin{proposition}[Regret Bound for \cref{proj_sub_alg}, {\cite[Thm. $3.1$]{hazan2016introduction}}]
\label{proj_sub_alg_regret}
The iterations $(v_k)_{k = 1}^K$ returned by \cref{proj_sub_alg} satisfy the following
linear regret bound w.r.t. a competitor vector $v \in \B$
\begin{equation}
\sum_{k = 1}^K \langle g_k, v_k - v \rangle \le 
C \sqrt{2} ~ {\rm diam}(\B) \sqrt{K}.
\end{equation}
\end{proposition}

We now have all the ingredients necessary to introduce the family
of within-task algorithms.

%---------------------------------------------------------------------------------------------------------

\section{BIASED PARAMETER-FREE ONLINE ALGORITHM}
\label{within_task_family}

In this section, we consider a family of within-task algorithms 
parametrized by a bias vector $\theta \in \Real^d$. The idea of introducing a bias 
is a well-established approach in the multi-task learning and meta-learning literature, 
see e.g. \cite{balcan2019provable,cavallanti2010linear,denevi2019learning,denevi2018learning,denevi2019online,evgeniou2005learning,khodak2019adaptive,maurer2006,pentina2014pac}. However, a key novel aspect of our work is to focus on a family of parameter-free algorithms -- we are not aware of previous work dealing with a similar framework within the multi-task learning or meta-learning literature. Such a choice allows us to avoid expensive validation procedures which are not even allowed in the so-called `adversarial setting', where the learner is asked to make predictions on the fly, after observing data only once. 
Specifically, the algorithm we choose is reported in  \cref{algorithm_free_bias} and it coincides with a variant of the online parameter-free \cite[Alg. $2$]{cutkosky2018black} in which we add a translation w.r.t. a bias vector $\theta \in \Real^d$, which is specified in advanced to the algorithm.

\begin{algorithm}[t]
\caption{{\fontsize{9pt}{10pt}\selectfont Parameter-Free Algorithm with Fixed Bias, Biased Version of {\cite[Alg. $2$]{cutkosky2018black}}}} \label{algorithm_free_bias}
{\fontsize{9pt}{10pt}\selectfont
\begin{algorithmic}
\State
\State {\bfseries Input} $\theta \in \Real^d$, $\Zn = (z_i)_{i = 1}^n = (x_i, y_i)_{i = 1}^n$, $e > 0$, $L$ and $\rx$ as in \cref{ass1}
\vspace{.1cm}
\State {\bfseries Initialize} $b_1 = 0$, $u_1 = e$, $p_1 = b_1 u_1$, 
$v_1 = 0 \in \B(0,1)$
%$b_1 = 0 \in \Real$ magnitude's betting fraction, $u_1 = e$ 
%magnitude's wealth, $p_1 = b_1 u_1$ magnitude, $v_1 = 0 \in \B(0,1)$ 
%direction %$w_1 = p_1 v_1 + \theta$ global vector
\vspace{.1cm}
\State {\bfseries For} $i = 1, \dots, n$
%\vspace{.2cm}
%\State ~~ \ds{$\#~1.$ Updating the global vector $w$}
\vspace{.1cm}
\State ~ 1.~ Vector update $w_i = p_i v_i + \theta$
%\vspace{2.cm}
%\State ~~ \ds{$\#~2.$ Receiving the new data-point and computing the new subgradient}
\vspace{.1cm}
\State ~ 2a. Receive the datapoint $z_i = (x_i, y_i)$
\vspace{.1cm}
\State ~ 2b. Compute $g_i = s_i x_i$, $s_i \in \partial \ell_i ( \langle x_i, w_i \rangle) 
\in \Real$
%\vspace{.2cm}
%\State ~~ \ds{$\#~3.$ Updating the direction $v$}
\vspace{.1cm}
\State ~ 3a. Define $\gamma_i = \frac{1}{L \rx} \sqrt{\frac{2}{i}}$
\vspace{.1cm}
\State ~ 3b. Direction update $v_{i+1} = \proj_{\B(0,1)}\bigl( v_i - 
\gamma_i g_i \bigr)$
%\vspace{.2cm}
%\State ~~ \ds{$\#~4.$ Updating the magnitude $p$}
\vspace{.1cm}
\State ~ 4a. Define $u_{i+1} 
%= \epsilon - \frac{1}{\rx L} \sum_{j = 1}^i \langle g_j, v_j \rangle p_j
= u_i - \frac{1}{\rx L} \langle g_i, v_i \rangle p_i$ 
\vspace{.1cm}
\State ~ 4b. Define $b_{i+1} 
%= - \frac{1}{\rx L i} \sum_{j = 1}^i \langle g_j, v_j \rangle
= \frac{1}{i} \big ( (i-1) b_i - \frac{1}{\rx L} \langle g_i, v_i \rangle \big )$ 
\vspace{.1cm}
\State ~ 4c. Magnitude update $p_{i+1} = b_{i+1} u_{i+1}$
\vspace{.1cm}
\State {\bfseries End}
\vspace{.1cm}
\State {\bfseries Return} $(w_i)_{i = 1}^n$
\end{algorithmic}}
\end{algorithm}

Similarly to the discussion 
in \cite{cutkosky2018black}, the motivation behind the algorithm comes from the following simple
observation. For a fixed bias vector $\theta \in \Real^d$, we can always 
rewrite any vector $w \in \Real^d$ w.r.t. the coordinate system centered in $\theta$:
\begin{equation} \label{parametrization_single}
w = p v + \theta
\end{equation}
where
\begin{equation} \label{parametrization_single_2}
p = \| w - \theta \| \in \Real
\quad \quad 
 v = \frac{w - \theta}{\| w - \theta \|} \in \B(0,1).
\end{equation}
\cref{algorithm_free_bias} receives in input the bias vector $\theta \in \Real^d$
and, exploiting the decomposition above, it uses the datapoints $\Zn = (z_i)_{i = 1}^n 
= (x_i, y_i)_{i = 1}^n$ it receives, in order to incrementally learn
\begin{itemize}
\item the direction $v \in \B(0,1)$ of the vector $w - \theta$ by applying 
\cref{proj_sub_alg} on the ball $\B(0,1)$ to the subgradient vectors 
$(g_i)_{i = 1}^n$, where $g_i \in \partial \ell_i ( \langle x_i, \cdot \rangle)(w_i)$, 
with $w_i$ the current global vector returned by the algorithm (steps $3a$-$b$),
\item the magnitude $p$ of the vector $w - \theta$ by 
applying \cref{coin_betting_alg} to the scalars $(\langle g_i, v_i \rangle)_{i = 1}^n$,
with $v_i$ the current direction w.r.t. the bias vector $\theta$ 
(steps $4a$-$c$).
\end{itemize} 

We remark that, in order to update the magnitude $p$, differently 
from \cref{algorithm_free_bias} in which we use the KT 
algorithm in \cref{coin_betting_alg}, the authors in \cite[Alg. $2$]{cutkosky2018black} use a more sophisticate coin betting algorithm based on online Newton step, see \cite[Alg. $1$]{cutkosky2018black}. This allows them to get more refined regret bounds, 
which can bring an advantage for instance in the smooth setting. 
In this work, we employ a simplified version of the algorithm for the theoretical
analysis, since the derived regret bounds are simpler and, at the same time, such 
a simplification does not affect the main message we want to convey.

In the following result we report a regret bound for 
\cref{algorithm_free_bias}. We make no claim of originality in the proof of the above result, which is a simple adaptation of the proof technique  
of \cite[Thm. $2$]{cutkosky2018black}, adding the translation w.r.t. the
bias and changing the coin betting algorithm to estimate the magnitude,
as explained above. We provide below the main ideas used in the proof of the statement
because they will be used also in the following. The full proof is reported 
in \cref{proof_regret_single} for completeness.

\begin{restatable}[Single-Task Regret Bound for \cref{algorithm_free_bias}, Adaptation of {\cite[Thm. $2$]{cutkosky2018black}}]{proposition}{RegretSingle}
\label{regret_single}
Let \cref{ass1} hold and let $(w_i)_{i = 1}^{n}$ be the iterates generated  by \cref{algorithm_free_bias} with bias $\theta \in \Real^d$. Then, for any 
$w \in \Real^d$,
\begin{equation*}
\begin{split}
& \sum_{i = 1}^n \ell_i(\langle x_i, w_i \rangle) - \ell_i(\langle x_i, w \rangle) 
\le \sum_{i = 1}^n \big \langle g_i, w_i - w \big \rangle \\
& ~ \le \rx L  \Bigg [ e + \Bigg ( 2 \sqrt{2} + \Phi \Big( e^{-1} \| w - \theta \| n \Big) \Bigg ) \| w - \theta \| \sqrt{n} \Bigg ]
\end{split}
\end{equation*}
where, $\Phi(\cdot)$ is defined as in \cref{coin_betting_alg_regret}.
\end{restatable}

\begin{proofsketch}
While the first inequality follows by the convexity of the loss function (see \cref{ass1}) 
and the definition of the subgradients $(g_i)_{i = 1}^n$, the proof of the second 
inequality is essentially based on the magnitude-direction decomposition used in the 
algorithm and explained above. Specifically, by definition of the $w_i$ in \cref{algorithm_free_bias} and the rewriting of $w \in \Real^d$ as in \cref{parametrization_single}--\cref{parametrization_single_2}, one can 
show that the linear regret can be bounded by the sum of two terms,
\begin{equation} 
\sum_{i = 1}^n \big \langle g_i, w_i - w \big \rangle
\le R(p) + p R( v),
\end{equation}
where, $p$ and $v$ are defined in \cref{parametrization_single_2} and 
\begin{equation*}
R(p) = \sum_{i = 1}^n \big \langle g_i, v_i \big \rangle ( p_i - p ) \quad 
R(v) = \sum_{i = 1}^n \big \langle g_i, v_i - v \big \rangle
\end{equation*}
coincide, respectively, with the regret of the magnitudes $(p_i)_{i = 1}^n$
generated by \cref{coin_betting_alg} and the regret of the directions 
$(v_i)_{i = 1}^n$ generated by \cref{proj_sub_alg}. The statement then 
follows from exploiting \cref{ass1} in order to bound the two terms by \cref{coin_betting_alg_regret} and \cref{proj_sub_alg_regret}, respectively.
\end{proofsketch}

As observed in \cite{cutkosky2018black}, the leading term in the bound 
above is equivalent -- up to the logarithmic factor contained in the term 
$\Phi(e^{-1} \| w - \theta \| n)$ -- to the optimal bound $\mathcal{O}
(\| w - \theta \| \sqrt{n})$ one would get by using a translated version 
of online subgradient algorithm
\begin{equation} \label{translated_algo_tuning}
w_1 = 0 \in \Real^d \quad 
w_i = w_{i -1} - \gamma g_{i-1} + \theta \quad i \ge 2, 
\end{equation}
with oracle-tuning of the step-size $\gamma > 0$ requiring knowledge 
of the target vector's magnitude $\| w - \theta \|$ in hindsight. Moreover, 
since the bound matches available lower-bounds, such additional logarithmic 
terms are unavoidable and they represent the price we pay by estimating the 
magnitude from the data. We also notice that, by induction argument it is easy 
to show that the translated iteration in \cref{translated_algo_tuning} coincides 
with standard (untranslated) online subgradient algorithm with initial point 
$\theta$. As a consequence, \cref{algorithm_free_bias} 
can be also interpreted as a parameter-free variant of the standard family used in 
fine tuning meta-learning \cite{finn2019online}.

%---------------------------------------------------------------------------------------------------------

\section{MOTIVATION FOR THE BIAS}
\label{motivation_bias}

In this section, we study the advantage of using an appropriate bias term in \cref{algorithm_free_bias}. Specifically, we study the performance obtained 
by applying \cref{algorithm_free_bias} with the same bias vector 
$\theta$ over a sequence of $T$ datasets ${\bf \Zn} = (\Zn_t)_{t = 1}^T$, 
$\Zn_t = (z_{t,i})_{i = 1}^n = (x_{t,i}, y_{t,i})_{i = 1}^n$ deriving 
from $T$ different tasks w.r.t. a sequence of target vectors $(w_t)_{t = 1}^T$
associated to the tasks. We are implicitly parametrizing the 
vector associated to each task as in \cref{parametrization_single}--\cref{parametrization_single_2}, according to the same bias vector 
$\theta \in \Real^d$:
\begin{equation} \label{parametrization_single_t}
w_t = p_t v_t + \theta
\end{equation}
\begin{equation} \label{parametrization_single_t2}
p_t = \| w_t - \theta \| \in \Real \quad 
v_t = \frac{w_t - \theta}{\| w_t - \theta \|} \in \B(0,1).
\end{equation}
This situation is analyzed below.

\begin{corollary}[Across-Tasks Regret Bound for \cref{algorithm_free_bias}]
\label{regret_across}
Let \cref{ass1} hold. Consider $T$ datasets ${\bf \Zn} = (\Zn_t)_{t = 1}^T$, $\Zn_t = (z_{t,i})_{i = 1}^n = (x_{t,i}, y_{t,i})_{i = 1}^n$ deriving from $T$ different tasks.
For any task $t = 1, \dots, T$, let $(w_{t,i})_{i = 1}^{n}$ be the iterates generated by \cref{algorithm_free_bias} over the dataset $\Zn_t$ with bias $\theta$. Then, for any sequence $(w_t)_{t = 1}^T$, $w_t \in \Real^d$,
\begin{eqnarray} 
\nonumber
& \sum\limits_{t = 1}^T \sum\limits_{i = 1}^n \ell_{t,i}( \langle x_{t,i}, w_{t,i} \rangle)
- \ell_{t,i}( \langle x_{t,i}, w_t \rangle) \quad \quad \quad \\ \nonumber
& \le \sum\limits_{t = 1}^T \sum\limits_{i = 1}^n \big \langle g_{t,i}, w_{t,i} - w_t \big \rangle  \quad \quad \quad \quad \quad\\  \label{termA_erased}
& \leq \rx L \Bigg [ e T + \Big (2 \sqrt{2} {\rm Var} (\theta) + \widehat {\rm Var} (\theta) \Big ) \sqrt{n} T \Bigg ]  \quad
\end{eqnarray}
where 
\begin{equation} \label{variance}
{\rm Var} (\theta) = \frac{1}{T} \sum_{t = 1}^T \| w_t - \theta \|,
\end{equation}
%\begin{equation} \label{termA}
%\text{A} = \rx L \Bigg [ e T + \Big (2 \sqrt{2} {\rm Var}_{{\rm MTL}}(\theta) + \widehat {\rm Var}_{{\rm MTL}}(\theta) \Big ) \sqrt{n} T \Bigg ] 
%\end{equation}
\begin{equation}
\widehat {\rm Var} (\theta) =
 \frac{1}{T} \sum_{t = 1}^T 
\Phi \Big( e^{-1} \| w_t - \theta \| n \Big) \| w_t - \theta \|
\end{equation}
and the function $\Phi(\cdot)$ is defined in \cref{coin_betting_alg_regret}.
\end{corollary}

\begin{proof}
The statement directly derives from summing over the datasets
the regret bound in \cref{regret_single}.
\end{proof}

Even though the quantity in \cref{variance} does not coincide 
with the variance of the target vectors $(w_t)_{t = 1}^T$ w.r.t. 
the bias $\theta$, with some abuse of notation, we are denoting
such a quantity by ${\rm Var}(\theta)$. To be precise, \cref{variance}
represents a lower-bound for the variance, indeed, by Jensen's inequality 
we have that  ${\rm Var}(\theta) \le \sqrt{\frac{1}{T} \sum_{t = 1}^T
\| w_t - \theta \|^2}$. We observe that, when the logarithmic term 
is negligible (i.e. $\phi(e^{-1} \| w_t - \theta \| n) \approx 1$) 
%\gd{Explain this.}
for any task $t \in \{1, \dots, T \}$, then $\widehat {\rm Var}(\theta) \approx {\rm Var}(\theta)$. As a consequence, in such a case, the leading term in the bound above is proportional to  
\begin{equation}
\mathcal{O}\Big ( {\rm Var}(\theta) \sqrt{n} T \Big ).
\end{equation}
The conclusion we get from the bound above is exactly in line with previous 
literature addressing the same problem, but by means of methods 
requiring the tuning of at least one hyper-parameter, see e.g.
\cite{balcan2019provable,denevi2019learning,denevi2019online,khodak2019adaptive}. 
Specifically, the bound above suggests that the optimal choice 
for the bias $\theta$ in \cref{algorithm_free_bias}
is the one minimizing the variance of the target vectors $(w_t)_{t = 1}^T$,
namely, their empirical average:
\begin{equation} \label{oracle}
\argmin_{\theta \in \Real^d} \sqrt{\frac{1}{T} \sum_{t = 1}^T \| w_t - \theta \|^2} = \frac{1}{T} \sum_{t = 1}^T w_t.
\end{equation}
Moreover, our analysis confirms the conclusion in 
\cite{balcan2019provable,denevi2019learning,denevi2019online,khodak2019adaptive}: the advantage of using such an optimal bias w.r.t. the unbiased case 
(corresponding to solving the tasks independently) is significant when the 
variance of the tasks' target vectors is much smaller than their second moment. 
%Motivated by these observations, in the next section we address the issue of learning the bias as the tasks are incrementally observed. 

%-------------------------------------------------------------------------------------------------------

\section{LEARNING THE BIAS}
\label{estimation_bias}

\begin{algorithm}[t]
\caption{{\fontsize{9pt}{10pt}\selectfont Parameter-Free Algorithm with Bias Inferred from Data, Aggressive Version}} \label{algorithm_free_bias_continuous_setting}
{\fontsize{9pt}{10pt}\selectfont
\begin{algorithmic}
\State
\State {\bfseries Input} ${\bf \Zn} = (\Zn_t)_{t = 1}^T$, $\Zn_t = (z_{t,i})_{i = 1}^n = (x_{t,i}, y_{t,i})_{i = 1}^n$, $e > 0$, $E > 0$, $L$ and $\rx$ as in \cref{ass1}
%${\bf \Zn} = (\Zn_t)_{t = 1}^T$, $\Zn_t = (z_{t,i})_{i = 1}^n = (x_{t,i}, y_{t,i})_{i = 1}^n$ datasets, $e > 0$ and $E > 0$ initial wealths, $L > 0$ Lipschitz constant of $\ell_{t,i}(\cdot)$ and $\rx > 0$ such that $\| x_{t,i} \| \le \rx$ for any $i = 1, \dots, n$ and $t = 1, \dots, T$
\vspace{.1cm}
\State {\bfseries Initialize} 
$B_{1} = 0$, $U_{1} = E$, 
$P_{1} = B_{1} U_{1}$, 
$V_{1} = 0 \in \B(0,1)$
%$B_{k(1,1)} = 0 \in \Real$ meta-magnitude's betting fraction, 
%$U_{k(1,1)} = E$ meta-magnitude's wealth, 
%$P_{k(1,1)} = B_{k(1,1)} U_{k(1,1)}$ meta-magnitude, 
%$V_{k(1,1)} = 0 \in \B(0,1)$ meta-direction
%%$\theta_{k(1,1)} = P_{k(1,1)} V_{k(1,1)}$ global meta-vector (bias)
\vspace{-.2cm}
\State {\bfseries For} $t = 1, \dots, T$
\vspace{.1cm}
\State ~ Set $b_{t,1} = 0$, $u_{t,1} = e$,
$p_{t,1} = b_{t,1} u_{t,1}$, $v_{t,1} = 0 \in \B(0,1)$
%\State ~ Define $b_{t,1} = 0 \in \Real$ within-task magnitude's betting fraction, 
%$u_{t,1} = e$ within-task magnitude's wealth, 
%$p_{t,1} = b_{t,1} u_{t,1}$ within-task magnitude, 
%$v_{t,1} = 0 \in \B(0,1)$ within-task direction
\vspace{-.2cm}
\State ~ {\bfseries For} $i = 1, \dots, n$
\vspace{.1cm}
\State ~~ 0. Define $k = k(t,i) = (t-1)n + i$
\vspace{.1cm}
%\State ~~ \ds{$\#~1. \text{(META)}$ Updating the global meta-vector $\theta$}
%\vspace{.1cm}
\State ~~ 1. Meta-vector update $\theta_{k} = P_{k} V_{k}$
\vspace{.1cm}
%\State ~~ \ds{$\#~1. \text{(WITHIN)}$ Updating the global within-task vector $w$}
%\vspace{.1cm}
\State ~~ 1. Within-vector update $w_{t,i} = p_{t,i} v_{t,i} + \theta_{k}$
\vspace{.1cm}
%\State ~~ \ds{$\#~2. \text{(META $\&$ WITHIN)}$ Receiving the new data-point and computing the new subgradient}
%\vspace{.1cm}
\State ~~ 2a. Receive the datapoint $z_{t,i} = (x_{t,i}, y_{t,i})$
\vspace{.1cm}
\State ~~ 2b. Compute $g_{t,i} = s_{t,i} x_{t,i}$, $s_{t,i} \in \partial \ell_{t,i} ( \langle x_{t,i}, w_{t,i} \rangle)$
%\vspace{.2cm}
%\State ~~ \ds{$\#~3. \text{(META)}$ Updating the meta-direction $V$}
\vspace{.1cm}
\State ~~ 3A. Define $\eta_{k} = \frac{1}{L \rx} \sqrt{\frac{2}{k}}$
\vspace{.1cm}
\State ~~ 3B. Define $V_{k+1} = \proj_{\B(0,1)}\bigl( V_{k} - \eta_{k} g_{t,i} \bigr)$
%\vspace{.2cm}
%\State ~~ \ds{$\#~3. \text{(WITHIN)}$ Updating the within-task direction $v$}
\vspace{.1cm}
\State ~~ 3a. Define $\gamma_{t,i} = \frac{1}{L \rx} \sqrt{\frac{2}{i}}$
\vspace{.1cm}
\State ~~ 3b. Update $v_{t,i+1} = \proj_{\B(0,1)}\bigl( v_{t,i} - \gamma_{t,i} g_{t,i} \bigr)$
%\vspace{.2cm}
%\State ~~ \ds{$\#~4. \text{(META)}$ Updating the meta-magnitude $P$}
\vspace{.1cm}
\State ~~ 4A. Define $U_{k+1}
%= E - \frac{1}{\rx L} \sum_{m = 1}^t \sum_{j = 1}^{i} \langle g_{m,j}, V_{k(m,j)} \rangle P_{k(m,j)}
= U_{k} - \frac{1}{\rx L} \langle g_{t,i}, V_{k} \rangle P_{k}$
\vspace{.1cm}
\State ~~ 4B. Define $B_{k+1} 
%= - \frac{1}{\rx L k(t,i)} \sum_{m = 1}^t \sum_{j = 1}^{i} \langle g_{m,j}, V_{k(m,j)} \rangle$
%\State ~~~~~~ 
%~~~~~~~~~~~~~~~~~~~~~~~~~~~~~~~~~~~~
%~~~~~~~~~~~~~~~~~~
= \frac{1}{k} \big ( ( k-1 ) B_{k} - \frac{1}{\rx L} \langle g_{t,i}, V_{k} \rangle \big )$ 
\vspace{.1cm}
\State ~~ 4C. Update $P_{k+1} = B_{k+1} U_{k+1}$
%\vspace{.2cm}
%\State ~~ \ds{$\#~4. \text{(WITHIN)}$ Updating the within-task magnitude $p$}
\vspace{.1cm}
\State ~~ 4a. Define $u_{t,i+1}
%= e - \frac{1}{\rx L} \sum_{j = 1}^i \langle g_{t,j}, v_{t,j} \rangle p_{t,j}
= u_{t,i} - \frac{1}{\rx L} \langle g_{t,i}, v_{t,i} \rangle p_{t,i}$
\vspace{.1cm}
\State ~~ 4b. Define $b_{t,i+1} 
%= - \frac{1}{\rx L i} \sum_{j = 1}^i \langle g_{t,j}, v_{t,j} \rangle
= \frac{1}{i} \big ( (i-1) b_{t,i} - \frac{1}{\rx L} \langle g_{t,i}, v_{t,i} \rangle \big )$ 
\vspace{.1cm}
\State ~~ 4c. Update $p_{t,i+1} = b_{t,i+1} u_{t,i+1}$
\vspace{.1cm}
\State ~ {\bfseries End}
\vspace{.1cm}
\State {\bfseries End}
\vspace{.1cm}
\State {\bfseries Return} $(w_{t,i})_{t = 1, i = 1}^{T,n}$ and $(\theta_{k})_{k =1}^{Tn}$
%\vspace{.1cm}
%\State \quad \quad \ds{In statistical MTL setting: $(\bar w_t)_{t = 1}^T$, $\bar w_t = \sum_{i = 1}^n w_{t,i}$}
\end{algorithmic}}
\end{algorithm}

Motivated by the conclusion in the previous section, we now propose and 
analyze a parameter-free method to infer a good bias vector shared across 
the tasks from an increasing sequence of $T$ datasets ${\bf \Zn} = 
(\Zn_t)_{t = 1}^T$, $\Zn_t = (z_{t,i})_{i = 1}^n = (x_{t,i}, y_{t,i})_{i = 1}^n$. 
The method is reported in \cref{algorithm_free_bias_continuous_setting}
and it updates the bias vector $\theta$ after each point. This characteristic 
inspires us to refer to \cref{algorithm_free_bias_continuous_setting}
as `aggressive', in order to distinguish it from the `lazy' version 
reported in \cref{algorithm_free_bias_continuous_setting_lazy}
in \cref{lazy_version}, where, the bias vector is updated only at
the end of each task. 

The idea motivating the design of our method is similar to 
the idea in \cref{motivation_bias} of parametrizing the vector associated 
to each task as in \cref{parametrization_single_t}--\cref{parametrization_single_t2}, 
according to a common bias vector $\theta \in \Real^d$.
But, now, we also parametrize the bias vector $\theta$ w.r.t. the 
zero-centered coordinate system:
\begin{equation} \label{parametrization_meta}
\theta = P V 
\end{equation}
\begin{equation} \label{parametrization_meta_2}
P = \| \theta \| \in \Real \quad \quad 
V = \frac{\theta}{\| \theta \|} \in \B(0,1).
\end{equation}
%in order to infer it from the data.
\cref{algorithm_free_bias_continuous_setting} exploits the \emph{joint} 
parametrization above and it uses the datasets ${\bf \Zn}$ it receives,
in order to incrementally learn
\begin{itemize}
\item (for any task $t$) the direction $v_t \in \B(0,1)$ of the vector $w_t - \theta$ by applying \cref{proj_sub_alg} on the ball $\B(0,1)$ to the subgradient vectors $(g_{t,i})_{i = 1}^n$, where $g_{t,i} \in \partial \ell_{t,i} ( \langle x_{t,i}, \cdot \rangle)(w_{t,i})$, with $w_{t,i}$ the current within-task iteration returned by the algorithm (steps $3a$-$b$),
\item (for any task $t$) the magnitude $p_t$ of the vector $w_t - \theta$, by 
applying \cref{coin_betting_alg} to the scalars $(\langle g_{t,i}, v_{t,i} \rangle)_{i = 1}^n$, with $v_{t,i}$ the current within-task direction w.r.t. the current bias vector $\theta_{k(t,i)}$ estimated by the algorithm and $k(t,i) = (t-1)n + i$ the total number of points seen up to that moment (steps $4a$--$c$),
\item the direction $V \in \B(0,1)$ of the vector $\theta$, by applying \cref{proj_sub_alg} on the ball $\B(0,1)$ to the vectors $(g_{t,i})_{t,i = 1}^{T,n}$ 
(steps $3A$-$B$),
\item the magnitude $P$ of the vector $\theta$, by applying \cref{coin_betting_alg} to the scalars $(\langle g_{t,i}, V_{k(t,i)} \rangle)_{t = 1, i = 1}^{T,n}$, with $V_{k(t,i)}$ the current meta-direction estimated by the algorithm (steps $4A$-$C$).
\end{itemize} 

The performance of \cref{algorithm_free_bias_continuous_setting}
is analyzed in the following theorem in which we give an across-tasks
regret bound for the method. The complete proof of the statement is 
reported in \cref{proof_regret_across_meta}.

\begin{restatable}[Across-Tasks Regret Bound for \cref{algorithm_free_bias_continuous_setting}]{theorem}{RegretAcrossMeta}
\label{regret_across_meta}
Let \cref{ass1} hold.
Consider $T$ datasets ${\bf \Zn} = (\Zn_t)_{t = 1}^T$, $\Zn_t = (z_{t,i})_{i = 1}^n = (x_{t,i}, y_{t,i})_{i = 1}^n$ deriving from $T$ different tasks. 
Let $(w_{t,i})_{t = 1, i = 1}^{T,n}$ be the iterates generated by \cref{algorithm_free_bias_continuous_setting} over these
datasets ${\bf \Zn}$. Then, for any sequence $(w_t)_{t = 1}^T$,
$w_t \in \Real^d$ and any $\theta \in \Real^d$,
\begin{equation}
\begin{split}
& \sum_{t = 1}^T \sum_{i = 1}^n \ell_{t,i}( \langle x_{t,i}, w_{t,i} \rangle)
- \ell_{t,i}( \langle x_{t,i}, w_t \rangle) \\
& \quad \le \sum_{t = 1}^T \sum_{i = 1}^n \big \langle g_{t,i}, w_{t,i} - w_t \big \rangle \le \text{A} + \text{B},
\end{split}
\end{equation}
where, A is the term in \cref{termA_erased} with $\theta$,
\begin{equation*}
\text{B} = \rx L \Bigg[ E + \Bigg ( 2 \sqrt{2} + \Phi \Big( E^{-1} \| \theta \| n T \Big) \Bigg) \| \theta \| \sqrt{n T} \Bigg ]
\end{equation*}
and $\Phi(\cdot)$ is defined as in \cref{coin_betting_alg_regret}.
\end{restatable}

\begin{proofsketch}
While the first inequality is due to the convexity of the loss function (see \cref{ass1}) and the definition
of the subgradients $(g_{t,i})_{t = 1, i = 1}^{T,n}$, the proof of the 
second inequality is based, also in this case, on the joint magnitude-direction 
decomposition motivating the design of the algorithm and explained above.
Specifically, by definition of $w_{t,i}$ and $\theta_k$ in \cref{algorithm_free_bias_continuous_setting} and the rewriting of $\theta$ as in \cref{parametrization_meta}--\cref{parametrization_meta_2} and $w_t$ as
in \cref{parametrization_single_t}--\cref{parametrization_single_t2},
one can show that the linear regret can be bounded by four 
contributions as follows:
\begin{equation*}
\begin{split} 
\sum_{t = 1}^T \sum_{i = 1}^n \big \langle g_{t,i}, w_{t,i} - w_t \big \rangle 
\le & \sum_{t = 1}^T \Big ( R_t(p_t) + p_t R_t(v_t) \Big ) \\
& + R(P) + P R(V),
\end{split}
\end{equation*}
where, $p_t$ and $v_t$ as in \cref{parametrization_single_t2},
$P$ and $V$ as in \cref{parametrization_meta_2}, 
\begin{equation}
R_t(p_t) = \sum_{i = 1}^n \big \langle g_{t,i}, v_{t,i} \big \rangle ( p_{t,i} - p_t )
\end{equation}
\begin{equation}
R_t(v_t) = \sum_{i = 1}^n \big \langle g_{t,i}, v_{t,i} - v_t \big \rangle
\end{equation}
\begin{equation}
R(P) = \sum_{t = 1}^T \sum_{i = 1}^n \big \langle g_{t,i}, V_{k(t,i)} \big \rangle \bigl( P_{k(t,i)} - P \bigr)
\end{equation}
\begin{equation}
R( V) = \sum_{t = 1}^T \sum_{i = 1}^n \big \langle g_{t,i}, V_{k(t,i)} 
- V \big \rangle
\end{equation}
coincide, respectively, with the within-task regret of the magnitudes $(p_{t,i})_{i = 1}^n$ generated by \cref{coin_betting_alg} on the task $t$, the within-task regret of the directions $(v_{t,i})_{i = 1}^n$ generated by \cref{proj_sub_alg} on the task $t$,
the meta-regret of the magnitudes $(P_{k})_{k = 1}^K$ generated by \cref{coin_betting_alg} and the meta-regret of the directions $(V_k)_{k = 1}^K$ 
generated by \cref{proj_sub_alg}. The statement derives from exploiting \cref{ass1} 
in order to bound the four terms by \cref{coin_betting_alg_regret} or \cref{proj_sub_alg_regret}, accordingly.
\end{proofsketch}

We note that the bound in \cref{regret_across_meta} is composed of two main terms: while the term $A$ coincides with the bound in \cref{termA_erased} for the use of a pre-fixed bias $\theta$ across all the tasks, the term $B$ captures the price we pay to estimate the bias from data. Notice that this additional term goes as $\mathcal{O}(\sqrt{n T})$ and, as a consequence, it is negligible when added to the first term going as $\mathcal{O}(\sqrt{n})$. 
In particular, specifying the bound in \cref{regret_across_meta} to the bias
$\theta$ in \cref{oracle} (the average of the target tasks' weight vectors), 
we can conclude that our method is able to match the performance of this best bias 
in hindsight, when the number of tasks is sufficiently large. On the other hand, by 
taking $\theta = 0 \in \Real^d$ in \cref{regret_across_meta}, we retrieve the bound 
in \cref{regret_across} for independent task learning (ITL). Hence, in the 
worst-case scenario of no low-variance tasks, our method performs, at least, as ITL,
without negative transfer effect. 
We also observe that the additional term due to the estimation of the bias from the data is faster in comparison to the additional term going as $\mathcal{O}(n \sqrt{T})$ paid in benchmark works for growing tasks' sequences requiring hyper-parameter tuning, such as \cite{denevi2019online}. This is essentially due to the fact that in our method we are updating the bias more frequently: after each point (hence $nT$ updates) instead of only at the end of each task (hence $T$ updates) as done in \cite{denevi2019online}. We finally notice that the bound in \cref{regret_across_meta} present a similar rate to the mistakes' bound in \cite[Cor. 4]{cavallanti2010linear} for a Perceptron-based algorithm. However, the method in \cite{cavallanti2010linear} works only for finite sequences of tasks and, again, it requires hyper-parameter tuning.
%\gd{What happens when the variance is 0 to the bound and the method?
%Comparison to the Srebo's method concatenating all the points?}

%--------------------------------------------------------------------------------------------------------

\section{STATISTICAL MTL SETTING}
\label{statistical_setting}

In this section we show how \cref{algorithm_free_bias_continuous_setting}
can be adapted to a multi-task learning statistical setting. 
Specifically, following the framework outlined in \cite{caruana1998multitask}, we assume that, for any $t \in \{1, \dots, T \}$, the within-task dataset $\data_t$ 
is an independently identically distributed (i.i.d.) sample from a distribution (task) 
$\task_t$. 

In this case, for any task $t \in \{1, \dots, T \}$, we consider the estimator 
$\bar w_t = \frac{1}{n} \sum_{i = 1}^n w_{t,i}$ given by the average of the 
iterations computed by \cref{algorithm_free_bias_continuous_setting} associated 
to the task $t$. We wish to study the performance of such estimators. Formally, 
for any task $\task_t$, we require that the corresponding true 
risk $\cR_{\task_t}(w) = \Exp_{(x,y) \sim \task_t} \ell ( \langle x, w \rangle, y )$ 
admits minimizers over the entire space $\Real^d$ and we denote by $\wmu_t$ 
the minimum norm one. With these ingredients, we introduce the multi-task 
{\rm oracle} $\ee_{{\rm MTL}}^* = \frac{1}{T} \sum_{t = 1}^T \cR_{\task_t}(\wmu_t)$, 
and, introducing the \emph{average multi-task risk} of the estimators 
$(\bar{w}_t)_{t = 1}^T$:
\begin{equation} \label{MTL_risk}
\E_{{\rm MTL}} \big ( (\bar{w}_t)_{t = 1}^T \big) = \frac{1}{T} \sum_{t = 1}^T \cR_{\task_t} ( \bar{w}_t ),
\end{equation}
we give a bound on it w.r.t. the oracle $\ee_{{\rm MTL}}^*$. This is described in the following theorem.

\begin{restatable}[Multi-Task Risk Bound for \cref{algorithm_free_bias_continuous_setting}]{theoremshortref}{ExcessRiskBound} 
\label{excess_risk_bound}
Let the same assumptions in \cref{regret_across_meta} hold in the i.i.d. multi-task statistical setting. Let $\E_{{\rm MTL}} \big ( (\bar{w}_t)_{t = 1}^T \big)$ be as in \cref{MTL_risk}, namely, the average multi-task risk of the 
estimators $(\bar{w}_t)_{t = 1}^T$, where $\bar w_t$ is the average of the 
iterates computed by \cref{algorithm_free_bias_continuous_setting} associated to 
the task $t$. Then, for any $\theta \in \Real^d$, 
in expectation w.r.t. the sampling of the datasets ${\bf Z} = (\Zn_t)_{t = 1}^T$,
\begin{equation}
\Exp_{\bf Z}~\E_{{\rm MTL}} \big ( (\bar{w}_t)_{t = 1}^T \big) - \ee_{{\rm MTL}}^* \le \frac{1}{n T} \bigl( \text{A} + \text{B} \bigr)
\end{equation}
where,
\begin{equation*} %\label{termA_stat}
\text{A} = \rx L \Bigg [ e T + \Big (2 \sqrt{2} {\rm Var}_{{\rm MTL}}(\theta) 
+ \widehat {\rm Var}_{{\rm MTL}}(\theta) \Big ) \sqrt{n} T \Bigg ] 
\end{equation*}
\begin{equation}
{\rm Var}_{{\rm MTL}}(\theta) = \frac{1}{T} \sum_{t = 1}^T \| w_{\task_t} - \theta \|
\end{equation}
\begin{equation*}
\widehat {\rm Var}_{{\rm MTL}}(\theta) = \frac{1}{T} \sum_{t = 1}^T
\Phi \Big( e^{-1} \| \hat w_{\task_t} - \theta \| n \Big) \| w_{\task_t} - \theta \|
\end{equation*}
%\begin{equation*}
%\text{B} = \rx L \Bigg[ E + \Bigg ( 2 \sqrt{2} + \Phi \Big( E^{-1} \| \theta \| n T \Big) \Bigg) \| \theta \| \sqrt{n T} \Bigg ]
%\end{equation*}
$\Phi(\cdot)$ is defined as in \cref{coin_betting_alg_regret}
and B is the term in \cref{regret_across_meta}.
\end{restatable}
The bound we have obtained above for our parameter-free method is in line 
with previous batch multi-task learning literature \cite{evgeniou2005learning,maurer2006} requiring 
tuning of hyper-parameters. Regarding online benchmarks, it is unclear
whether the online method proposed in \cite{cavallanti2010linear} can be adapted 
also to a statistical setting.
We observe that the bound above is composed by the expectation of the terms comparing in \cref{regret_across_meta} evaluated at the target vectors $(w_{\task_t})_{t = 1}^T$. This automatically derives from the fact that the proof of the statement exploits the across-tasks regret bound given in \cref{regret_across_meta} for our meta-learning procedure and, as described in the following proposition, online-to-batch conversion arguments \cite{cesa2004generalization,littlestone2014line}. The statement reported below is used by the authors in \cite{eichner2019semi} in order to address the issue of 
applying stochastic subgradient descent to a sequence of semy-cyclic datapoints by 
plurastic (multi-task) point of view. We report the proof in \cref{online_to_batch_proof} for completeness.
 
\begin{restatable}[Online-To-Batch Conversion for  \cref{algorithm_free_bias_continuous_setting}, {\cite[Thm. $3$]{eichner2019semi}}]{proposition}{OnlineToBatch} 
\label{online_to_batch}
Under the same assumptions in \cref{excess_risk_bound}, the following relation holds
\begin{equation*}
\begin{split}
& \Exp_{\bf Z}~\E_{{\rm MTL}} \big ( (\bar{w}_t)_{t = 1}^T \big) - \ee_{{\rm MTL}}^* \le \\
& \quad \Exp_{\bf Z}~ \Bigg [ \frac{1}{nT} \sum_{t = 1}^T \sum_{i = 1}^n \ell_{t,i}( \langle x_{t,i}, w_{t,i} \rangle) - \ell_{t,i}( \langle x_{t,i}, w_{\task_t} \rangle) \Bigg ].
\end{split}
\end{equation*}
\end{restatable}

We now have all the ingredients necessary to prove \cref{excess_risk_bound}.

\begin{proofGD}{\bf of \cref{excess_risk_bound}.}
The desired statement derives from applying on the right side of \cref{online_to_batch}
the across-tasks regret bound in \cref{regret_across_meta} specified to the 
sequence of target vectors $(w_{\task_t})_{t = 1}^T$.
\end{proofGD}

Looking at the proof in \cref{online_to_batch_proof}, the reader can 
notice that the online-to-batch statement in \cref{online_to_batch}
applies also to the `lazy' version of our method reported in \cref{algorithm_free_bias_continuous_setting_lazy} in \cref{lazy_version}.
This allows us to convert also the lazy variant into a statistical 
(sub-optimal) multi-task learning method with a slower rate. 
On the other hand, we did not manage to convert the aggressive 
variant of our method into a statistical meta-learning method.
This issue makes us wondering whether faster rates going as 
$\sqrt{n T}$ as in the multi-task learning setting are achievable 
also in the meta-learning setting for the second term.

%--------------------------------------------------------------------------------------------------------

\section{EXPERIMENTS}
\label{experiments}

In this section we test the numerical performance of our method\footnote{Code to reproduce the experiments is available at {\em https://github.com/dstamos/Parameter-free-MTL}}.
Following the same data-generation procedure described
in \cite{denevi2019learning}, we generated an environment 
of $T = 400$ regression tasks with low variance. Specifically, for any 
task $\task$, we sampled the corresponding ground truth vector $w_\task$ 
from a Gaussian distribution with mean given by the vector 
$\theta^* \in \mathbb{R}^d$ with $d = 10$ and all components equal
to $4$ and standard deviation $1$. After this, we generated the corresponding 
dataset $(x_i,y_i)_{i=1}^n$, $x_i\in\mathbb{R}^d$ with $n = 25$. We sampled 
the inputs uniformly on the unit sphere and we generated the labels according to the equation $y = \langle x,\wmu \rangle + \epsilon$, where the noise $\epsilon$ was sampled from a zero-mean Gaussian distribution, with standard deviation chosen in 
order to have signal-to-noise ratio $1$. 

\begin{figure}[t]
\begin{minipage}[t]{0.49\textwidth}  
\centering
\includegraphics[width=1\textwidth]{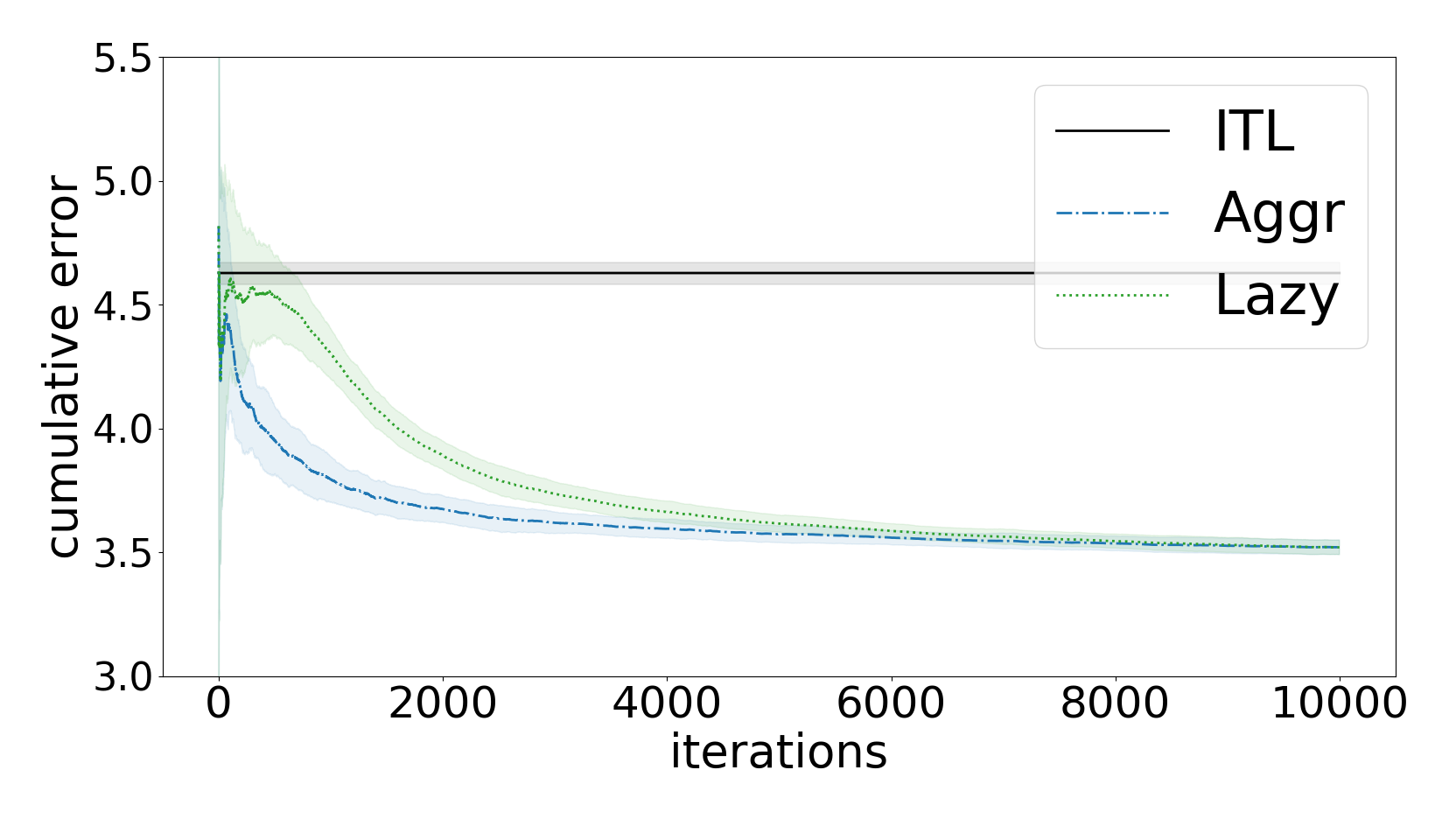}
%\vspace{-.3cm}
\end{minipage}
\begin{minipage}[t]{0.49\textwidth}
\centering
\includegraphics[width=1\textwidth]{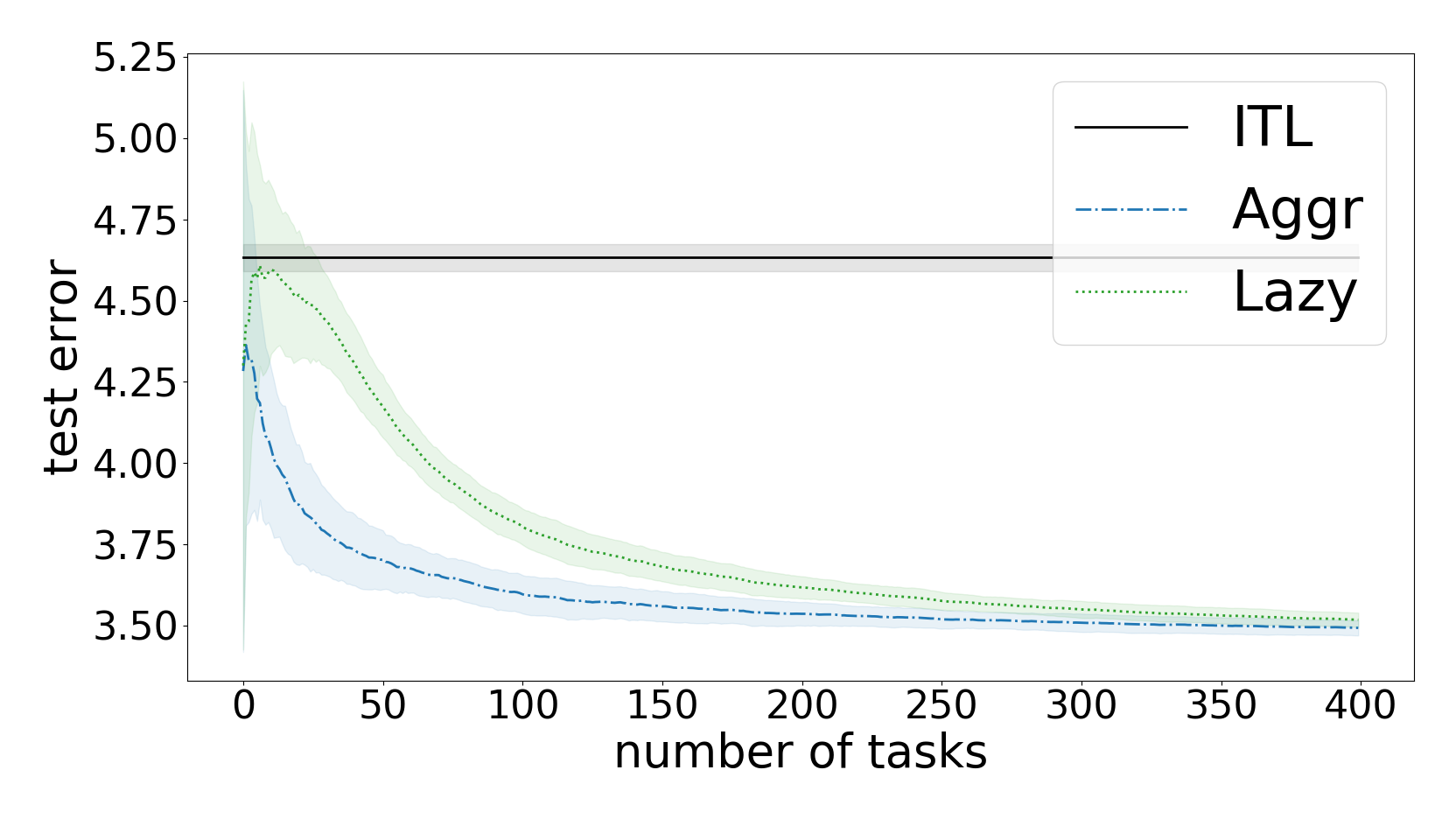}
%\vspace{-.3cm}
\end{minipage}
\vspace{-.5cm}
\caption{Average performance (over $30$ seeds) 
of different methods w.r.t. an increasing number of 
iterations or tasks on synthetic data. Average 
across-tasks cumulative error (top), average multi-task 
test error (bottom). \label{fig_exps_real_synth}}
\vspace{-.5cm}
\end{figure}

In this setting, we compared the performance of independent task learning 
(ITL) (running the unbiased variant of \cref{algorithm_free_bias} over each task), 
%the oracle (running \cref{algorithm_free_bias} with the continual variant 
%$\theta^* \in \Real^d$ of the bias in \cref{oracle} over each task), 
our aggressive method in \cref{algorithm_free_bias_continuous_setting} 
(Aggr) and its lazy version in \cref{algorithm_free_bias_continuous_setting_lazy} 
in \cref{lazy_version} (Lazy). 
%We added also to the comparison
%the method applying the unbiased version of \cref{algorithm_free_bias} 
%across the entire ordered stream of points $(z_k)_{k = 1}^K$,
%where, we recall that $k = k(t,i) = (t-1)n + i$ and $K = Tn$.

In the experiments below, we noticed that the variants of our methods estimating the magnitude by the refined coin betting algorithm in \cite[Alg. $1$]{cutkosky2018black}
returned a more readable plot 
%converged faster 
w.r.t. the variants described in our theory using the KT algorithm in
\cref{coin_betting_alg}. For this reason, we report 
below the results obtained using this more refined variant. 

In \cref{fig_exps_real_synth} (top) we report the average 
across-tasks cumulative error for all the methods w.r.t. to an 
increasing number of datapoints/iterations. In \cref{fig_exps_real_synth} 
(bottom) we report their (statistical) average multi-task test errors 
for an increasing number of tasks. We measured the performance by the 
absolute loss and we set the initial wealths in our methods equal to $1$, for 
both the within-task and the across-tasks algorithms. The results we got are in 
agreement with the theory. Our approaches lead to a substantial benefits 
w.r.t. ITL and they converge to the oracle (the algorithm with the best bias 
in hindsight) as the number of the observed datapoints/tasks increases. 
Moreover, coherently with our bounds, we observe that, the aggressive variant 
of our method presents faster rates w.r.t. its lazy counterpart.
%\gd{Finally, we observe that the naive method ... outperforms both our 
%methods. This is not surprising, since such a method enjoys a faster
%convergence rate proportional to $\mathcal{O}(\sqrt{n T})$.}

Because of lack of space, in \cref{exp_details}, we report additional 
experiments investigating the sensitivity of our parameter-free methods 
w.r.t. to the initialization of the wealths and showing the effectiveness of 
our methods on two real datasets (the Lenk \cite{lenk1996hierarchical,Andrew} 
and the Schools \cite{argyriou2008convex} datasets). In such a case, we 
will report for completeness both the refined and the basic variant of our methods.

%------------------------------------------------------------------------------------------------------

\section{CONCLUSION}
\label{conclusion}

We developed a parameter-free method that learns a common bias 
shared by a growing sequence of tasks. The advantage of
our method in comparison to solving the tasks independently manifests itself 
when the variance of target tasks' weight vectors is sufficiently small. Our 
method is originally introduced in the non-statistical setting and it can be applied 
into an aggressive or lazy version. The aggressive version enjoys faster rates
and it can be converted into a statistical multi-task learning method, while, 
the lazy method recovers standard rates, but it can be converted into
a statistical meta-learning method, able to generalize across the tasks.

In the future it would valuable to investigate whether other multi-task 
learning methods based on different metrics and addressing different 
types of tasks' relatedness (e.g. those based on a shared low dimensional 
representation \cite{denevi2018incremental,tripuraneni2020provable} 
or graph regularization \cite{cavallanti2010linear,evgeniou2005learning}) 
can be made parameter-free as well. 
Moreover, it would be interesting to understand if our analysis allows  
more cycles over the data as in \cite{eichner2019semi} and if this can be 
beneficial. Finally, we also wonder whether our parameter-free approach 
can be beneficial to recent meta-learning frameworks \cite{cella2020meta} 
dealing with partial feedback scenarios \cite{altschuler2018online}.
%Finally, we wonder whether faster rates going as $\sqrt{n T}$ as in the 
%multi-task setting are achievable also in the meta-learning setting.

%\gd{Do we want to observe that, by concatenating all weight vectors in-.035t a unique 
%larger vector, we can make a more general parameter-free method but only for 
%MTL (not LTL)?}

%-------------------------------------------------------------------------------------------------------

\subsubsection*{Acknowledgements}

This work was supported in part by SAP SE and by EPSRC Grant N. EP/P009069/1.

\bibliographystyle{abbrv}

%--------------------------------------------------------------------------------------------------------

\onecolumn
\appendix

\section*{APPENDIX}

The appendix is structured in the following way. We start from 
reporting the proof of \cref{regret_single} in \cref{proof_regret_single}.
After that, we give the proof of \cref{regret_across_meta} and 
\cref{online_to_batch} in \cref{proof_regret_across_meta} and
\cref{online_to_batch_proof}, respectively. Then, in \cref{lazy_version},
we present and analyze the lazy variant of our method giving rise
to a parameter-free statistical meta-learning method. Finally, in 
\cref{exp_details}, we report additional experiments we omitted in 
the main body because of lack of space.

%--------------------------------------------------------------------------------------------------------

\section{PROOF OF \cref{regret_single}}
\label{proof_regret_single}

\RegretSingle*

\begin{proof}
We start by observing that the first inequality in the statement holds by convexity
of $\ell_i(\langle x_i, \cdot \rangle)$ (see \cref{ass1}) and the fact that, by construction, $g_i \in \partial
\ell_i(\langle x_i, \cdot \rangle)(w_i)$. In order to show the second inequality, we just
proceed as in the proof of \cite[Thm. $2$]{cutkosky2018black}. Specifically, by definition of $w_i$ in \cref{algorithm_free_bias} and the rewriting of $w$ as in  \cref{parametrization_single}--\cref{parametrization_single_2}, we can write
\begin{equation} \label{initial_decomposition_single}
\begin{split}
\sum_{i = 1}^n \big \langle g_i, w_i - w \big \rangle
& = \sum_{i = 1}^n \big \langle g_i, p_i v_i + \theta - (p v + \theta) \big \rangle \\
& = \sum_{i = 1}^n \big \langle g_i, p_i v_i - p v \big \rangle \\
& = \sum_{i = 1}^n \big \langle g_i, p_i v_i - p v \big \rangle 
\pm \big \langle g_i, p v_i \big \rangle \\
& = \sum_{i = 1}^n \big \langle g_i, v_i \big \rangle ( p_i - p ) 
+ p \sum_{i = 1}^n \big \langle g_i, v_i - v \big \rangle.
\end{split}
\end{equation}
In order to get the desired statement, we bound the two terms above as follows.
We start from observing that the first term above coincides with the regret of the 
sequence of scalars $(p_i)_{i = 1}^n$ generated by \cref{coin_betting_alg} aiming 
at inferring the magnitude $p$ from the sequence of scalars $(\langle g_i, v_i \rangle)_{i = 1}^n$. We also notice that, since by construction $\| v_i \| \le 1$ and since 
the Lipschitz and bounded inputs assumption in \cref{ass1} implies $\| g_i \| 
= | s_i | \| x_i \| \le L \rx$ (see \cite[Lemma $14.7$]{shalev2014understanding}), 
then we have
\begin{equation}
\big | \langle g_i, v_i \rangle \big | \le \| g_i \| \| v_i \| \le \rx L.
\end{equation}
As a consequence, recalling the initial wealth $e  > 0$ of the algorithm, 
by \cref{coin_betting_alg_regret}, we have, 
\begin{equation} \label{first_term_single}
\sum_{i = 1}^n \big \langle g_i, v_i \big \rangle \bigl( p_i - p \bigr)
 \le \rx L \Big [ e + \Phi \big ( e^{-1} p n \big ) p \sqrt{n} \Big ].
\end{equation}
Regarding the second term, we observe that the quantity
$\sum_{i = 1}^n \big \langle g_i, v_i - v \big \rangle$
coincides with the regret of the sequence $(v_i)_{i = 1}^n$ generated by 
\cref{proj_sub_alg} on $\B(0,1)$ aiming at inferring the direction $v$ 
from the sequence of vectors $(g_i)_{i = 1}^n$, where, as observed above,
by the Lipschitz and bounded inputs assumption in \cref{ass1}, we have
$\| g_i \| \le \rx L$.
As a consequence, by \cref{proj_sub_alg_regret}, we have, 
\begin{equation} \label{second_term_single}
\sum_{i = 1}^n \big \langle g_i, v_i - v \big \rangle
\le \rx L 2 \sqrt{2} \sqrt{n}.
\end{equation}
Substituting \cref{first_term_single} and \cref{second_term_single} into \cref{initial_decomposition_single}, we get the desired statement, recalling
that $p = \| w - \theta \|$.
\end{proof}

%--------------------------------------------------------------------------------------------------

\section{PROOF OF \cref{regret_across_meta}}
\label{proof_regret_across_meta}

\RegretAcrossMeta*

\begin{proof}
The first inequality in the statement holds by convexity
of $\ell_{t,i}(\langle x_{t,i}, \cdot \rangle)$ (see \cref{ass1}) and the fact 
$g_{t,i} \in \partial \ell_{t,i}(\langle x_{t,i}, \cdot \rangle)(w_{t,i})$. In order to show the second inequality, we proceed similarly to the proof of \cref{regret_single}, but, we now take into account also the variation of the bias across the iterations. Specifically, by definition of $w_{t,i}$ and $\theta_k$ in \cref{algorithm_free_bias_continuous_setting} and the rewriting of $\theta$ in \cref{parametrization_meta}--\cref{parametrization_meta_2} and $w_t$ in
\cref{parametrization_single_t}--\cref{parametrization_single_t2}, we can write
\begin{equation} \label{initial_decomposition}
\begin{split}
\sum_{t = 1}^T \sum_{i = 1}^n \big \langle g_{t,i}, w_{t,i} - w_t \big \rangle
& = \sum_{t = 1}^T \sum_{i = 1}^n \big \langle g_{t,i}, p_{t,i} v_{t,i} + \theta_{k(t,i)} - ( p_t v_t + \theta) \big \rangle \\
& = \sum_{t = 1}^T \sum_{i = 1}^n \big \langle g_{t,i}, p_{t,i} v_{t,i} - p_t v_t \big \rangle + \big \langle g_{t,i}, \theta_{k(t,i)} - \theta \big \rangle \\
& = \sum_{t = 1}^T \sum_{i = 1}^n p_ {t,i} \big \langle g_{t,i}, v_{t,i} \big \rangle - \big \langle g_{t,i}, p_t v_t \big \rangle + \big \langle g_{t,i}, \theta_{k(t,i)} - \theta \big \rangle \pm p_t \big \langle g_{t,i}, v_{t,i} \big \rangle \\
& = \sum_{t = 1}^T \sum_{i = 1}^n \big \langle g_{t,i}, v_{t,i} \big \rangle \bigl( p_{t,i} - p_t \bigr) + \sum_{t = 1}^T p_t \sum_{i = 1}^n \big \langle g_{t,i}, v_{t,i} - v_t \big \rangle \\
& \quad + \sum_{t = 1}^T \sum_{i = 1}^n \big \langle g_{t,i}, \theta_{k(t,i)} - \theta \big \rangle \\
& = \sum_{t = 1}^T \sum_{i = 1}^n \big \langle g_{t,i}, v_{t,i} \big \rangle \bigl( p_{t,i} - p_t \bigr) + \sum_{t = 1}^T p_t \sum_{i = 1}^n \big \langle g_{t,i}, v_{t,i} - v_t \big \rangle \\
& \quad + \sum_{t = 1}^T \sum_{i = 1}^n \big \langle g_{t,i}, P_{k(t,i)} V_{k(t,i)} - P V \big \rangle \\
& = \sum_{t = 1}^T \sum_{i = 1}^n \big \langle g_{t,i}, v_{t,i} \big \rangle \bigl( p_{t,i} - p_t \bigr) + \sum_{t = 1}^T p_t \sum_{i = 1}^n \big \langle g_{t,i}, v_{t,i} - v_t \big \rangle \\
& \quad + \sum_{t = 1}^T \sum_{i = 1}^n \big \langle g_{t,i}, P_{k(t,i)} V_{k(t,i)} - P V \big \rangle \pm P \langle g_{t,i}, V_{k(t,i)} \big \rangle \\
& = \sum_{t = 1}^T \sum_{i = 1}^n \big \langle g_{t,i}, v_{t,i} \big \rangle \bigl( p_{t,i} - p_t \bigr) + \sum_{t = 1}^T p_t \sum_{i = 1}^n \big \langle g_{t,i}, v_{t,i} - v_t \big \rangle \\
& \quad + \sum_{t = 1}^T \sum_{i = 1}^n \big \langle g_{t,i}, V_{k(t,i)} \big \rangle \bigl( P_{k(t,i)} - P \bigr) 
+ P \sum_{t = 1}^T \sum_{i = 1}^n \big \langle g_{t,i}, V_{k(t,i)} - V \big \rangle.
\end{split}
\end{equation}
In order to get the desired statement, we bound all the four terms above as follows.
Regarding the first term, for any task $t \in \{ 1, \dots, T \}$, we observe that the quantity
$\sum_{i = 1}^n \big \langle g_{t,i}, v_{t,i} \big \rangle \bigl( p_{t,i} - p_t \bigr)$
coincides with the regret of the sequence of scalars $(p_{t,i})_{i = 1}^n$, generated by \cref{coin_betting_alg} aiming at inferring the within-task magnitude $p_t$ from the sequence of scalars $(\langle g_{t,i}, v_{t,i} \rangle)_{i = 1}^n$. We also notice that, since by construction $\| v_{t,i} \| \le 1$ and since the Lipschitz and bounded inputs assumption in \cref{ass1} implies $\| g_{t,i} \| = | s_{t,i} | \| x_{t,i} \| \le L \rx$ (see \cite[Lemma $14.7$]{shalev2014understanding}), then we have
\begin{equation}
\big | \langle g_{t,i}, v_{t,i} \rangle \big | \le \| g_{t,i} \| \| v_{t,i} \| \le \rx L.
\end{equation}
As a consequence, recalling the initial within-task wealth $e  > 0$ of the algorithm, by
\cref{coin_betting_alg_regret}, we have, 
\begin{equation} \label{first_term}
\sum_{i = 1}^n \big \langle g_{t,i}, v_{t,i} \big \rangle \bigl( p_{t,i} - p_t \bigr)
 \le \rx L \Big [ e + \Phi \big ( e^{-1} p_t n \big ) p_t \sqrt{n} \Big ].
\end{equation}
Regarding the second term, for any task $t \in \{ 1, \dots, T \}$, we observe that the quantity $\sum_{i = 1}^n \big \langle g_{t,i}, v_{t,i} - v_t \big \rangle$
coincides with the regret of the sequence $(v_{t,i})_{i = 1}^n$, generated by 
\cref{proj_sub_alg} on $\B(0,1)$ aiming at inferring the within-task direction $v_t$ 
from the sequence of vectors $(g_{t,i})_{i = 1}^n$, where, as observed above,
by Lipschitz and bounded inputs assumption in \cref{ass1}, we have
$\| g_{t,i} \| \le \rx L$. As a consequence, by \cref{proj_sub_alg_regret}, we have, 
\begin{equation} \label{second_term}
\sum_{i = 1}^n \big \langle g_{t,i}, v_{t,i} - v_t \big \rangle
\le \rx L 2 \sqrt{2} \sqrt{n}.
\end{equation}
We now observe that the third term
coincides with the regret of the sequence of scalars $(P_{k(t,i)})_{t = 1, i = 1}^{T,n}$, generated by \cref{coin_betting_alg} aiming at inferring the meta-magnitude $P$ from the sequence of scalars $(\langle g_{t,i}, V_{k(t,i)} \rangle)_{t =1, i = 1}^{T,n}$.
We also notice that, since by construction $\| V_{k(t,i)} \| \le 1$ and since 
the Lipschitz and bounded inputs assumption in \cref{ass1} implies $\| g_{t,i} \| 
= | s_{t,i} | \| x_{t,i} \| \le L \rx$ (see \cite[Lemma $14.7$]{shalev2014understanding}), 
then we have
\begin{equation}
\big | \langle g_{t,i}, V_{k(t,i)} \rangle \big | 
\le \| g_{t,i} \| \| V_{k(t,i)} \| \le \rx L.
\end{equation}
As a consequence, recalling the initial meta-wealth $E  > 0$ of the algorithm, 
by \cref{coin_betting_alg_regret}, we have, 
\begin{equation} \label{third_term}
\sum_{t = 1}^T \sum_{i = 1}^n \big \langle g_{t,i}, V_{k(t,i)} \big \rangle \bigl( P_{k(t,i)} - P \bigr) 
\le \rx L \Big [ E + \Phi \big ( E^{-1} P nT \big ) P \sqrt{n T} \Big ].
\end{equation}
Finally, we observe that the fourth term
coincides with the regret of the sequence $(V_{k(t,i)})_{t = 1, i = 1}^{T,n}$, generated by \cref{proj_sub_alg} on $\B(0,1)$ aiming at inferring the meta-direction $V$ 
from the sequence of vectors $(g_{t,i})_{t = 1, i = 1}^{T,n}$, where, as observed above, by Lipschitz and bounded inputs assumption in \cref{ass1}, $\| g_{t,i} \| \le \rx L$.
As a consequence, by \cref{proj_sub_alg_regret}, we have, 
\begin{equation} \label{fourth_term}
\sum_{t = 1}^T \sum_{i = 1}^n \big \langle g_{t,i}, V_{k(t,i)} - V \big \rangle
\le \rx L 2 \sqrt{2} \sqrt{n T}.
\end{equation}
Substituting \cref{first_term}, \cref{second_term}, \cref{third_term} and 
\cref{fourth_term} into \cref{initial_decomposition}, we get the desired statement,
once one recalls that $p_t = \| w_t - \theta \|$ and $P = \| \theta \|$.
\end{proof}

%--------------------------------------------------------------------------------------------------------

\section{Proof of \cref{online_to_batch}}
\label{online_to_batch_proof}

\OnlineToBatch*

\begin{proof}
During the proof we write explicitly the expectation
\begin{equation}
\Exp_{\bf Z} = \Exp_{{\Zn_1} \sim \task_1^n, \dots, {\Zn_T} \sim \task_T^n}.
\end{equation}
By definition of 
$\E_{{\rm MTL}} \big ( (\bar{w}_t)_{t = 1}^T \big)$
in \cref{MTL_risk}, we have that
\begin{equation} \label{first_a}
\begin{split}
\Exp_{{\Zn_1} \sim \task_1^n, \dots, {\Zn_T} \sim \task_T^n}~
\E_{{\rm MTL}} \big ( (\bar{w}_t)_{t = 1}^T \big) 
& = \Exp_{{\Zn_1} \sim \task_1^n, \dots, {\Zn_T} \sim \task_T^n}~\frac{1}{T} \sum_{t = 1} \cR_{\task_t}(\bar w_t) \\
& \le \Exp_{{\Zn_1} \sim \task_1^n, \dots, {\Zn_T} \sim \task_T^n}~\frac{1}{nT} \sum_{t = 1}^T \sum_{i = 1}^n \cR_{\task_t}(w_{t,i}) \\
%& = \Exp_{{\Zn_1} \sim \task_1^n, \dots, {\Zn_T} \sim \task_T^n}~\frac{1}{nT} \sum_{t = 1}^T \sum_{i = 1}^n \cR_{\task_t}(w_{t,i}) \\
& = \Exp_{{\Zn_1} \sim \task_1^n, \dots, {\Zn_T} \sim \task_T^n}~ \frac{1}{nT} \sum_{t = 1}^T \sum_{i = 1}^n \Exp_{(x,y) \sim \task_t}\ell_y(\langle x, w_{t,i} \rangle ) \\
& = \frac{1}{nT} \sum_{t = 1}^T \sum_{i = 1}^n \Exp_{{\Zn_1} \sim \task_1^n, \dots, {\Zn_{t-1}} \sim \task_{t-1}^n}~\Exp_{{(\hspace{-.032truecm}}z_{t,j}{\hspace{-.032truecm})}_{\hspace{-.02truecm}j{=}1}^{i{-}1} \sim \task_t^{i-1}}~\Exp_{z \sim \task_t}\ell_y(\langle x, w_{t,i} \rangle ) \\
& = \frac{1}{nT} \sum_{t = 1}^T \sum_{i = 1}^n \Exp_{{\Zn_1} \sim \task_1^n, \dots, {\Zn_T} \sim \task_{t-1}^n}~\Exp_{{(\hspace{-.032truecm}}z_{t,j}{\hspace{-.032truecm})}_{\hspace{-.02truecm}j{=}1}^{i} \sim \task_t^{i}}~\ell_{t,i}(\langle x_{t,i}, w_{t,i} \rangle ) \\
& = \Exp_{{\Zn_1} \sim \task_1^n, \dots, {\Zn_T} \sim \task_T^n}~ 
\frac{1}{nT} \sum_{t = 1}^T \sum_{i = 1}^n 
\ell_{t,i} ( \langle x_{t,i}, w_{t,i}\rangle),
\end{split}
\end{equation}
where, the first inequality follows by convexity of the function 
$\cR_{\task_t}$, in the third equality we have exploited the fact that $w_{t,i}$ depends only on the datasets $Z_1,\dots,Z_{t-1}$ and the first $i-1$ points of the $t$-th dataset, $(z_{t,j})_{j {=} 1}^{i{-}1}$, finally, in the fourth equality, since $z_{t,i} \sim \task_t$, we have used the identity
\begin{equation}
\Exp_{{(\hspace{-.032truecm}}z_{t,j}{\hspace{-.032truecm})}_{\hspace{-.02truecm}j{=}1}^{i{-}1} \sim \task_t^{i-1}}~\Exp_{(x,y) \sim \task_t}\ell_y(\langle w_{t,i}, x \rangle )
= \Exp_{{(\hspace{-.032truecm}}z_{t,j}{\hspace{-.032truecm})}_{\hspace{-.02truecm}j{=}1}^{i} \sim \task_t^{i}}~\ell_{t,i}(\langle w_{t,i}, x_{t,i} \rangle ).
\end{equation}
Next, because of the i.i.d. sampling of the data, we observe that we can write 
\begin{equation} \label{second_a}
\Exp_{{\Zn_1} \sim \task_1^n, \dots, {\Zn_T} \sim \task_T^n}~
\frac{1}{nT} \sum_{t = 1}^T \sum_{i = 1}^n 
\ell_{t,i} ( \langle x_{t,i}, w_{\task_t} \rangle)
= \Exp_{{\Zn_1} \sim \task_1^n, \dots, {\Zn_T} \sim \task_T^n}~
\frac{1}{T} \sum_{t = 1}^T \cR_{\task_t}(w_{\task_t}).
\end{equation}
The desired statement now follows from combining \cref{first_a}
and \cref{second_a}.
\end{proof}

As we already observed in the main body, looking at the proof above, 
we notice that the online-to-batch statement in \cref{online_to_batch}
applies also to the `lazy' version of our method reported in \cref{algorithm_free_bias_continuous_setting_lazy} in \cref{lazy_version}.
This allows us to convert also this lazy variant into a statistical (sub-optimal) 
multi-task learning method with a slower rate.

%--------------------------------------------------------------------------------------------------------

\section{LAZY VERSION OF OUR METHOD}
\label{lazy_version}

In this section, we present and analyze the lazy variant of our method 
introduced in the main body. Specifically, after introducing the lazy method
in \cref{algorithm_free_bias_continuous_setting_lazy}, we give an 
across-tasks regret bound for it in \cref{regret_section_lazy}. After 
that, in \cref{section_stat_meta_lazy}, we show how the method 
can be converted into a parameter-free statistical meta-learning method 
and we give a transfer risk bound for it.

\begin{algorithm}[t]
\caption{{\fontsize{9pt}{10pt}\selectfont Parameter-Free Algorithm with Bias Inferred from Data, Lazy Version}} \label{algorithm_free_bias_continuous_setting_lazy}
{\fontsize{9pt}{10pt}\selectfont
\begin{algorithmic}
\State
\State {\bfseries Input} ${\bf \Zn} = (\Zn_t)_{t = 1}^T$, $\Zn_t = (z_{t,i})_{i = 1}^n = (x_{t,i}, y_{t,i})_{i = 1}^n$, $e > 0$, $E > 0$, $L$ and $\rx$ as in \cref{ass1}
%${\bf \Zn} = (\Zn_t)_{t = 1}^T$, $\Zn_t = (z_{t,i})_{i = 1}^n = (x_{t,i}, y_{t,i})_{i = 1}^n$ datasets, $e > 0$ and $E > 0$ initial wealths, $L > 0$ Lipschitz constant of $\ell_{t,i}(\cdot)$ and $\rx > 0$ such that $\| x_{t,i} \| \le \rx$ for any $i = 1, \dots, n$ and $t = 1, \dots, T$
\vspace{.1cm}
\State {\bfseries Initialize} 
$B_{1} = 0$, $U_{1} = E$, 
$P_{1} = B_{1} U_{1}$, 
$V_{1} = 0 \in \B(0,1)$
%$B_{k(1,1)} = 0 \in \Real$ meta-magnitude's betting fraction, 
%$U_{k(1,1)} = E$ meta-magnitude's wealth, 
%$P_{k(1,1)} = B_{k(1,1)} U_{k(1,1)}$ meta-magnitude, 
%$V_{k(1,1)} = 0 \in \B(0,1)$ meta-direction
%%$\theta_{k(1,1)} = P_{k(1,1)} V_{k(1,1)}$ global meta-vector (bias)
\vspace{.1cm}
\State {\bfseries For} $t = 1, \dots, T$
\vspace{.1cm}
\State ~ 1. Meta-vector update $\theta_{t} = P_{t} V_{t}$
\vspace{.1cm}
\State ~ Set $b_{\theta_t,1} = 0$, $u_{\theta_t,1} = e$,
$p_{\theta_t,1} = b_{\theta_t,1} u_{\theta_t,1}$, $v_{\theta_t,1} = 0 \in \B(0,1)$
%\State ~ Define $b_{t,1} = 0 \in \Real$ within-task magnitude's betting fraction, 
%$u_{t,1} = e$ within-task magnitude's wealth, 
%$p_{t,1} = b_{t,1} u_{t,1}$ within-task magnitude, 
%$v_{t,1} = 0 \in \B(0,1)$ within-task direction
\vspace{.1cm}
\State ~ {\bfseries For} $i = 1, \dots, n$
\vspace{.1cm}
\State ~~ 0. Define $k = k(t,i) = (t-1)n + i$
\vspace{.1cm}
%\State ~~ \ds{$\#~1. \text{(META)}$ Updating the global meta-vector $\theta$}
%\vspace{.1cm}
%\State ~~ 1. Meta-vector update $\theta_{k} = P_{k} V_{k}$
%\vspace{.1cm}
%\State ~~ \ds{$\#~1. \text{(WITHIN)}$ Updating the global within-task vector $w$}
%\vspace{.1cm}
\State ~~ 1. Within-vector update $w_{\theta_t,i} = p_{\theta_t,i} v_{\theta_t,i} + \theta_{t}$
\vspace{.1cm}
%\State ~~ \ds{$\#~2. \text{(META $\&$ WITHIN)}$ Receiving the new data-point and computing the new subgradient}
%\vspace{.1cm}
\State ~~ 2a. Receive the datapoint $z_{t,i} = (x_{t,i}, y_{t,i})$
\vspace{.1cm}
\State ~~ 2b. Compute $g_{t,i} = s_{t,i} x_{t,i}$, $s_{t,i} \in \partial \ell_{t,i} ( \langle x_{t,i}, w_{\theta_t,i} \rangle)$
%\vspace{.2cm}
%\State ~~ \ds{$\#~3. \text{(META)}$ Updating the meta-direction $V$}
%\vspace{.1cm}
%\State ~~ 3a. Define $\eta_{k} = \sqrt{\frac{2}{L \rx k}}$
%\vspace{.1cm}
%\State ~~ 3b. Define $V_{k+1} = \proj_{\B(0,1)}\bigl( V_{k} - \eta_{k} g_{t,i} \bigr)$
%\vspace{.2cm}
%\State ~~ \ds{$\#~3. \text{(WITHIN)}$ Updating the within-task direction $v$}
\vspace{.1cm}
\State ~~ 3a. Define $\gamma_{t,i} = \frac{1}{L \rx } \sqrt{\frac{2}{i}}$
\vspace{.1cm}
\State ~~ 3b. Update $v_{\theta_t,i+1} = \proj_{\B(0,1)}\bigl( v_{\theta_t,i} - \gamma_{t,i} g_{t,i} \bigr)$
%\vspace{.2cm}
%\State ~~ \ds{$\#~4. \text{(META)}$ Updating the meta-magnitude $P$}
\vspace{.1cm}
%\State ~~ 4a. Define $U_{k+1}
%%= E - \frac{1}{\rx L} \sum_{m = 1}^t \sum_{j = 1}^{i} \langle g_{m,j}, V_{k(m,j)} \rangle P_{k(m,j)}
%= U_{k} - \frac{1}{\rx L} \langle g_{t,i}, V_{k} \rangle P_{k}$
%\vspace{.1cm}
%\State ~~ 4b. Define $B_{k+1} 
%%= - \frac{1}{\rx L k(t,i)} \sum_{m = 1}^t \sum_{j = 1}^{i} \langle g_{m,j}, V_{k(m,j)} \rangle$
%%\State ~~~~~~ 
%%~~~~~~~~~~~~~~~~~~~~~~~~~~~~~~~~~~~~
%%~~~~~~~~~~~~~~~~~~
%= - \frac{1}{k} \big ( ( k-1 ) B_{k} + \frac{1}{\rx L} \langle g_{t,i}, V_{k} \rangle \big )$ 
%\vspace{.1cm}
%\State ~~ 4c. Update $P_{k+1} = B_{k+1} U_{k+1}$
%\vspace{.2cm}
%\State ~~ \ds{$\#~4. \text{(WITHIN)}$ Updating the within-task magnitude $p$}
\vspace{.1cm}
\State ~~ 4a. Define $u_{\theta_t,i+1}
%= e - \frac{1}{\rx L} \sum_{j = 1}^i \langle g_{t,j}, v_{t,j} \rangle p_{t,j}
= u_{\theta_t,i} - \frac{1}{\rx L} \langle g_{t,i}, v_{\theta_t,i} \rangle p_{\theta_t,i}$
\vspace{.1cm}
\State ~~ 4b. Define $b_{\theta_t,i+1} 
%= - \frac{1}{\rx L i} \sum_{j = 1}^i \langle g_{t,j}, v_{t,j} \rangle
= \frac{1}{i} \big ( (i-1) b_{\theta_t,i} - \frac{1}{\rx L} \langle g_{t,i}, v_{\theta_t,i} \rangle \big )$ 
\vspace{.1cm}
\State ~~ 4c. Update $p_{\theta_t,i+1} = b_{\theta_t,i+1} u_{\theta_t,i+1}$
\vspace{.1cm}
\State ~ {\bfseries End}
\vspace{.1cm}
\State ~ 3A. Define $\eta_{t} = \frac{1}{L \rx n} \sqrt{\frac{2}{t}}$ and $G_t = \sum_{i = 1}^n g_{t,i}$
\vspace{.1cm}
\State ~ 3B. Define $V_{t+1} = \proj_{\B(0,1)}\bigl( V_{t} - \eta_{t} G_t \bigr)$
\vspace{.1cm}
\State ~ 4A. Define $U_{t+1}
= U_{t} - \frac{1}{\rx L n} \langle G_t, V_{t} \rangle P_{t}$
\vspace{.1cm}
\State ~ 4B. Define $B_{t+1} 
= \frac{1}{t} \big ( ( t-1 ) B_{t} - \frac{1}{\rx L n} \langle G_t, V_{t} \rangle \big )$ 
\vspace{.1cm}
\State ~ 4C. Update $P_{t+1} = B_{t+1} U_{t+1}$
\vspace{.1cm}
\State {\bfseries End}
\vspace{.1cm}
\State {\bfseries Return} $(w_{\theta_t,i})_{t = 1, i = 1}^{T,n}$ and $(\theta_{t})_{t =1}^{T}$
\end{algorithmic}}
\end{algorithm}

%-------------------------------------------------------------------------------------------------------

\subsection{ACROSS-TASKS REGRET BOUND}
\label{regret_section_lazy}

\cref{algorithm_free_bias_continuous_setting_lazy} contains 
the lazy version of \cref{algorithm_free_bias_continuous_setting}.
As already stressed in the main body, we call \cref{algorithm_free_bias_continuous_setting_lazy} `lazy' since,
differently from \cref{algorithm_free_bias_continuous_setting}, the
bias is updated only at the end of the task.
In the following theorem, we give an across-tasks regret bound 
for \cref{algorithm_free_bias_continuous_setting_lazy}.

\begin{restatable}[Across-Tasks Regret Bound for \cref{algorithm_free_bias_continuous_setting_lazy}]{theorem}{RegretAcrossMetaLazy}
\label{regret_across_meta_lazy}
Let \cref{ass1} hold.
Consider $T$ datasets ${\bf \Zn} = (\Zn_t)_{t = 1}^T$, $\Zn_t = (z_{t,i})_{i = 1}^n = (x_{t,i}, y_{t,i})_{i = 1}^n$ deriving from $T$ different tasks. 
Let $(w_{\theta_t,i})_{t = 1, i = 1}^{T,n}$ be the iterates generated by \cref{algorithm_free_bias_continuous_setting_lazy} over these
datasets ${\bf \Zn}$. Then, for any $w_t \in \Real^d$ and 
$\theta \in \Real^d$, 
\begin{equation}
\sum_{t = 1}^T \sum_{i = 1}^n \ell_{t,i}( \langle x_{t,i}, w_{\theta_t,i} \rangle)
- \ell_{t,i}( \langle x_{t,i}, w_t \rangle) 
\le \sum_{t = 1}^T \sum_{i = 1}^n \big \langle g_{t,i}, w_{\theta_t,i} - w_t \big \rangle \le \text{A} + \text{B},
\end{equation}
where, A is the term 
%in \cref{termA} 
in \cref{termA_erased},
\begin{equation}
\text{B} = \rx L \Bigg[ E n + \Bigg ( 2 \sqrt{2} + \Phi \Big( E^{-1} \| \theta \| T \Big) \Bigg) \| \theta \| n \sqrt{T} \Bigg ]
\end{equation}
and $\Phi(\cdot)$ is defined as in \cref{coin_betting_alg_regret}.
\end{restatable}

\begin{proof}
The proof exactly proceeds as the proof of \cref{regret_across_meta}.
The first inequality in the statement holds by convexity
of $\ell_{t,i}(\langle x_{t,i}, \cdot \rangle)$ (see \cref{ass1}) and the fact that, by construction, $g_{t,i} \in \partial \ell_{t,i}(\langle x_{t,i}, \cdot \rangle)(w_{\theta_t,i})$. We now proceed
with the proof of the second inequality. Specifically, by definition of $w_{\theta_t,i}$ and 
$\theta_t$ in \cref{algorithm_free_bias_continuous_setting_lazy}, the rewriting of $\theta$ in \cref{parametrization_meta}--\cref{parametrization_meta_2} and $w_t$ in
\cref{parametrization_single_t}--\cref{parametrization_single_t2},
proceeding in the same way as done in the proof of \cref{regret_across_meta}, 
one can show that the linear regret can be bounded by four contributions as follows:
\begin{equation} \label{initial_decomposition_lazy}
\begin{split}
\sum_{t = 1}^T \sum_{i = 1}^n \big \langle g_{t,i}, w_{\theta_t,i} - w_t \big\rangle 
& \le \sum_{t = 1}^T \sum_{i = 1}^n \big \langle g_{t,i}, v_{\theta_t,i} \big \rangle \bigl( p_{\theta_t,i} - p_t \bigr) + \sum_{t = 1}^T p_t \sum_{i = 1}^n \big \langle g_{t,i}, v_{\theta_t,i} - v_{t} \big \rangle \\
& \quad + \sum_{t = 1}^T \big \langle G_t, V_t \big \rangle \bigl( P_t - P \bigr) 
+ P \sum_{t = 1}^T \big \langle G_t, V_t - V \big \rangle.
\end{split}
\end{equation}
In order to get the desired statement, we bound all the four terms above as follows.
The first two terms can be bounded as described in the proof of \cref{regret_across_meta} in \cref{first_term} and \cref{second_term}. We now observe that the third term
coincides with the regret of the sequence of scalars $(P_t)_{t = 1}^{T}$, generated by \cref{coin_betting_alg} aiming at inferring the meta-magnitude $P$ from the sequence of scalars $(\langle G_t, V_t \rangle)_{t =1}^{T}$.
We also notice that, since by construction $\| V_t \| \le 1$ and since 
the Lipschitz and bounded inputs assumption in \cref{ass1} implies $\| g_{t,i} \| 
= | s_{t,i} | \| x_{t,i} \| \le L \rx$ (see \cite[Lemma $14.7$]{shalev2014understanding}), 
then we have
\begin{equation}
\big | \langle G_t, V_t \rangle \big | \le 
\| G_t \| \| V_t \| = 
\Bigg \| \sum_{i = 1}^n g_{t,i} \Bigg \| \| V_t \| 
\le \sum_{i = 1}^n \| g_{t,i} \| \| V_t \|  \le \rx L n.
\end{equation}
As a consequence, recalling the initial meta-wealth $E  > 0$ of the algorithm, 
by \cref{coin_betting_alg_regret}, we have, 
\begin{equation} \label{third_term_lazy}
\sum_{t = 1}^T \big \langle G_t, V_t \big \rangle \bigl( P_t - P \bigr) 
\le \rx L \Big [ E n + \Phi \big ( E^{-1} P T \big ) P n \sqrt{T} \Big ].
\end{equation}
Finally, we observe that the fourth term
coincides with the regret of the sequence $(V_t)_{t = 1}^{T}$, generated by \cref{proj_sub_alg} on $\B(0,1)$ aiming at inferring the meta-direction $V$ 
from the sequence of vectors $(G_t)_{t = 1}^{T}$, where, as observed above, by Lipschitz and bounded inputs assumption in \cref{ass1}, $\| G_t \| \le \rx L n$.
As a consequence, by \cref{proj_sub_alg_regret}, we have, 
\begin{equation} \label{fourth_term_lazy}
\sum_{t = 1}^T \big \langle G_t, V_t - V \big \rangle
\le \rx L 2 \sqrt{2} n\sqrt{T}.
\end{equation}
Substituting \cref{first_term}, \cref{second_term}, \cref{third_term_lazy} and 
\cref{fourth_term_lazy} into \cref{initial_decomposition_lazy}, we get the desired statement,
once one recalls that $p_t = \| w_t - \theta \|$ and $P = \| \theta \|$.
\end{proof}

We observe that the bound above is equivalent to the bound for the method presented in \cite{denevi2019online}, which requires, however, oracle tuning of two hyper-parameters. We observe also that the term A in the bound above is exactly equivalent to the term A in \cref{regret_across_meta} for the aggressive version of the algorithm. However, as expected, in this lazy version, the term B is slower: it goes as 
$\mathcal{O}(n\sqrt{T})$, instead of the $\mathcal{O}(\sqrt{n T})$ rate 
in \cref{regret_across_meta} for the aggressive variant.

%--------------------------------------------------------------------------------------------------------

\subsection{STATISTICAL META-LEARNING SETTING}
\label{section_stat_meta_lazy}

In this section we show how we can convert \cref{algorithm_free_bias_continuous_setting_lazy} into a parameter-free
statistical meta-learning algorithm and we present guarantees for it. 
We consider the statistical meta-learning framework described in \cite{baxter2000model,maurer2005algorithmic,maurer2016benefit}.
More precisely, we assume that, for any $t \in \{1, \dots, T \}$, the 
within-task dataset $\data_t$ is an independently identically distributed 
(i.i.d.) sample from a distribution (task) $\task_t$, and in turn the tasks 
$(\task_t)_{t = 1}^T$ are an i.i.d. sample from a meta-distribution (or
\emph{environment}) $\env$. 
Differently from the multi-task learning setting described in the main body, 
in the meta-learning setting here, we want to select an estimator which is
able to generalize also across the tasks.  

In this section, we will make explicit the dependency w.r.t. the dataset 
and the bias in the iteration generated by \cref{algorithm_free_bias}.
The estimator we consider here is $\bar{w}_{\theta_{\hat t}}(\Zn) = \frac{1}{n} \sum_{i = 1}^n w_{\theta_{\hat t},i}(\Zn)$, the average of the iterates resulting from applying \cref{algorithm_free_bias} to a test dataset $\Zn$ with bias $\theta_{\hat t}$, a vector uniformly sampled among the bias vectors returned by our meta-algorithm in \cref{algorithm_free_bias_continuous_setting_lazy} applied to the training datasets ${\bf Z} = (\Zn_t)_{t = 1}^T$. In this case, we want to study the performance of such an 
estimator in expectation w.r.t. the tasks sampled from the environment $\env$. 

Formally, for any $\task \sim \env$, we require that the corresponding true risk $\cR_\task(w) = \Exp_{(x,y) \sim \task} \ell ( \langle x, w \rangle, y )$ admits minimizers over the entire space $\Real^d$ and we denote by $\wmu$ the minimum norm one.
With these ingredients, we introduce the {\rm meta-learning oracle} $\ee_{{\rm META}}^* = \EE_{\task \sim \env}~\cR_\task(\wmu)$
and, introducing the \emph{transfer risk} of the estimator $\bar{w}_{\theta_{\hat t}}$:
\begin{equation} \label{transfer_risk_lazy}
\E_{{\rm META}} ( \bar{w}_{\theta_{\hat t}} ) = \Exp_{\task \sim \env}~\Exp_{\Zn \sim \task^n}~ \cR_\task ( \bar{w}_{\theta_{\hat t}} (\Zn) ),
\end{equation}
we give a bound on it w.r.t. the oracle $\ee_{{\rm META}}^*$.  This is described in the following theorem.

\begin{restatable}[Transfer Risk Bound for \cref{algorithm_free_bias_continuous_setting_lazy}]{theoremshortref}{ExcessRiskBoundLazy} 
\label{excess_risk_bound_lazy}
Let the same assumptions in \cref{regret_across_meta_lazy}
hold in the i.i.d. meta-learning statistical setting. Let $\E_{{\rm META}} ( \bar{w}_{\theta_{\hat t}} )$ as in \cref{transfer_risk_lazy}, namely,
the transfer risk of the average $\bar w_{\theta_{\hat t}}$ of the iterates 
generated by \cref{algorithm_free_bias} with bias $\theta_{\hat t}$ uniformly 
sampled among the bias vectors returned by \cref{algorithm_free_bias_continuous_setting_lazy} applied to the training 
datasets ${\bf Z} = (\Zn_t)_{t = 1}^T$. Then, for any 
$\theta \in \Real^d$, in expectation w.r.t. the sampling of the datasets 
${\bf Z}$ and the uniform sampling of $\hat t \sim \U(T)$,
\begin{equation}
\Exp_{\hat t \sim \U(T)}~\Exp_{\bf Z}~ \E_{{\rm META}} ( \bar{w}_{\theta_{\hat t}} ) - \ee_{{\rm META}}^* \le \frac{1}{n T} \bigl( \text{A} + \text{B} \bigr)
\end{equation}
where
\begin{equation} \label{termA_stat}
\text{A} = \rx L \Bigg [ e T + \Big (2 \sqrt{2} {\rm Var}_{{\rm META}}(\theta) + \widehat {\rm Var}_{{\rm META}}(\theta) \Big ) \sqrt{n} T \Bigg ],
\end{equation}
\begin{equation}
{\rm Var}_{{\rm META}}(\theta) 
= \Exp_{\task \sim \env} \| w_{\task} - \theta \|,
\end{equation}
\begin{equation}
\widehat {\rm Var}_{{\rm META}}(\theta) 
= \Exp_{\task \sim \env} \Big [ \Phi \Big( e^{-1} \| w_{\task} - \theta \| n \Big) \| w_{\task} - \theta \| \Big ],
\end{equation}
B is the term in \cref{regret_across_meta_lazy}
and $\Phi(\cdot)$ is defined as in \cref{coin_betting_alg_regret}.
\end{restatable}

We observe that the bound above is composed by the expectation of the terms comparing in \cref{regret_across_meta_lazy}. This automatically derives from the fact that, as we will see in the following, the proof of the statement exploits the across-tasks regret bound given in \cref{regret_across_meta_lazy} for \cref{algorithm_free_bias_continuous_setting_lazy}
and, as described in the following proposition, online-to-batch conversion arguments \cite{cesa2004generalization,littlestone2014line}.
 
\begin{restatable}[Online-To-Batch Conversion for  \cref{algorithm_free_bias_continuous_setting_lazy}]{proposition}{OnlineToBatchLazy} 
\label{online_to_batch_lazy}
Under the same assumptions in \cref{excess_risk_bound}, the following relation holds
\begin{equation}
\Exp_{\hat t \sim \U(T)}~\Exp_{\bf Z}~ \E_{{\rm META}} ( \bar{w}_{\theta_{\hat t}} ) - \ee_{{\rm META}}^* \le \Exp_{\bf Z}~ \Bigg [ \frac{1}{nT} \sum_{t = 1}^T \sum_{i = 1}^n \ell_{t,i}( \langle x_{t,i}, w_{\theta_t,i} \rangle) - \ell_{t,i}( \langle x_{t,i}, w_{\task_t} \rangle) \Bigg ].
\end{equation}
\end{restatable}

\begin{proof}
In the following, we will explicitly write the expectation $\Exp_{\bf Z}$ in 
the statement above as
\begin{equation}
\Exp_{\bf Z} =
\Exp_{\task_1, \dots, \task_T \sim \env^T} ~ \Exp_{{\Zn_1} \sim \task_1^n, \dots, {\Zn_T} \sim \task_T^n}.
\end{equation}
Writing more explicitly the expectation w.r.t. the uniform sampling
$\hat t \sim \U(T)$ and exploiting the definition of $\E_{{\rm META}}
( \bar{w}_{\theta_{\hat t}} )$, we can write the following
\begin{equation} \label{statement_online_to_batch_1}
\begin{split}
& \Exp_{\hat t \sim \U(T)}~\Exp_{\task_1, \dots, \task_T \sim \env^T} ~ \Exp_{{\Zn_1} \sim \task_1^n, \dots, {\Zn_T} \sim \task_T^n}~\E_{{\rm META}} ( \bar{w}_{\theta_{\hat t}} ) \\
& \quad = \Exp_{\hat t \sim \U(T)}~\Exp_{\task_1, \dots, \task_T \sim \env^T} ~ \Exp_{{\Zn_1} \sim \task_1^n, \dots, {\Zn_T} \sim \task_T^n}~\Exp_ {\task \sim \env} ~ \Exp_{\Zn \sim \task^n} ~ \cR_{\task} (\bar w_{\theta_{\hat t}}(\Zn)) \\
& \quad = \frac{1}{T}~\Exp_{\task_1, \dots, \task_T \sim \env^T} ~ \Exp_{{\Zn_1} \sim \task_1^n, \dots, {\Zn_T} \sim \task_T^n}~\sum_{t = 1}^T \Exp_ {\task \sim \env} ~ \Exp_{\Zn \sim \task^n} ~ \cR_{\task} (\bar w_{\theta_t}(\Zn)) \\
& \quad = \frac{1}{T}~\sum_{t = 1}^T \Exp_{\task_1, \dots, \task_{t-1} \sim \env^{t-1}} ~ \Exp_{{\Zn_1} \sim \task_1^n, \dots, {\Zn_{t-1}} \sim \task_{t-1}^n}~\Exp_ {\task_t \sim \env} ~ \Exp_{\Zn_t \sim \task_t^n} ~ \cR_{\task_t} (\bar w_{\theta_t}(\Zn_t)) \\
& \quad \le \frac{1}{n T}~\sum_{t = 1}^T \sum_{i = 1}^n \Exp_{\task_1, \dots, \task_{t-1} \sim \env^{t-1}} ~ \Exp_{{\Zn_1} \sim \task_1^n, \dots, {\Zn_{t-1}} \sim \task_{t-1}^n}~\Exp_ {\task_t \sim \env} ~ \Exp_{\Zn_t \sim \task_t^n} ~ \cR_{\task_t} ( w_{\theta_t,i}(\Zn_t)) \\
& \quad = \Exp_{\task_1, \dots, \task_T \sim \env^T} ~ \Exp_{{\Zn_1} \sim \task_1^n, \dots, {\Zn_T} \sim \task_T^n} ~ \frac{1}{n T} \sum_{t = 1}^T \sum_{i = 1}^n \ell_{t,i} ( \langle x_{t,i}, w_{\theta_t,i}(\Zn_t) \rangle )
\end{split}
\end{equation}
where, in the third equality we have exploited the fact that ${\theta_t}$ depends only on 
$(\Zn_j)_{j = 1}^{t-1}$ and the i.i.d. sampling of the datasets, in the inequality we have applied Jensen's inequality to the convex function $\cR_{\task_t}$ and, finally, in the last equality we have exploited the fact that $w_{\theta_t,i}(\Zn_t)$ depends only on the points $(z_{t,j})_{j = 1}^{i-1}$ and, consequently, thanks to the fact $\Zn_t \sim \task_t^n$,
\begin{equation}
\Exp_{{\Zn_t} \sim \task_t^n} ~ \cR_{\task_t} ( w_{\theta_t,i} (\Zn_t) ) 
= \Exp_{{\Zn_t} \sim \task_t^n} ~ \ell_{t,i} ( \langle 
x_{t,i}, w_{\theta_t,i}(\Zn_t) \rangle).
\end{equation}
We now observe also that, by the i.i.d. sampling of the training data, we can write
the following
\begin{equation} \label{statement_online_to_batch_2}
\ee_{{\rm META}}^* = \Exp_{\task \sim \env} \cR_\task(w_\task) =
\Exp_{\task_1, \dots, \task_T \sim \env^T} ~ \Exp_{{\Zn_1} \sim \task_1^n, \dots, {\Zn_T} \sim \task_T^n} ~ \frac{1}{nT} \sum_{t = 1}^T \sum_{i = 1}^n \ell_{t,i}( \langle x_{t,i}, w_{\task_t} \rangle).
\end{equation}
The desired statement derives from combining \cref{statement_online_to_batch_1}
and \cref{statement_online_to_batch_2}.
\end{proof}

We observe that the online-to-batch conversion above, similarly to \cite[Thm. $6.1$]{alquier2016regret} and \cite[Thm. $3.3$]{balcan2019provable}, holds for a meta-parameter randomly sampled from the pool. In practice, this means that, when the 
number of training tasks $T$ is not known a priori, the method requires keeping in memory the meta-parameters estimated during the training phase in order to perform
this sampling in the test phase. To give guarantees for an estimator which can be 
computed more efficiently by our method, as done in \cite{denevi2019learning,denevi2019online}, is still an open question.
As already pointed out in the main body, we also observe that
we did not manage to develop an online-to-batch conversion
similar to the one above in \cref{online_to_batch_lazy} for the 
aggressive variant of our method in \cref{algorithm_free_bias_continuous_setting}.
In other words, we did not know whether it is possible to convert 
the aggressive variant of our method into a statistical meta-learning 
method able to generalize also to new tasks. This would imply faster 
rates going as $\sqrt{n T}$ for the second term in the bounds also for the 
meta-learning setting.

We now have all the ingredients necessary for the proof of \cref{excess_risk_bound_lazy}.

\begin{proofGD}{\bf of \cref{excess_risk_bound_lazy}.} 
The desired statement derives from applying on the right side of \cref{online_to_batch_lazy}
the across-tasks regret bound in \cref{regret_across_meta_lazy} specified to the 
sequence of target vectors $(w_{\task_t})_{t = 1}^T$.
\end{proofGD}

%--------------------------------------------------------------------------------------------------------

\section{ADDITIONAL EXPERIMENTS}
\label{exp_details}

In this section we report additional experiments investigating 
the sensitivity w.r.t. the initial wealths and the effectiveness over real 
data of our methods. Also in these cases, we considered regression 
settings and we evaluated the errors by the absolute loss. In the plots 
below we reported also the (aggressive and lazy) variants of our 
parameter-free methods analyzed in our theory and using the KT  
algorithm in \cref{coin_betting_alg} to estimate the
magnitudes. We will denote these variants with the subscript `KT' 
to distinguish them from their counterparts estimating the magnitudes 
by the more refined variant of the coin betting algorithm described in 
\cite[Alg. $2$]{cutkosky2018black}.

%------------------------------------------------------------------------------------------------------

\subsection{SENSITIVITY W.R.T. THE INITIAL WEALTHS}

\begin{figure}[t]
\begin{minipage}[t]{0.49\textwidth}  
\centering
\includegraphics[width=.9\textwidth]{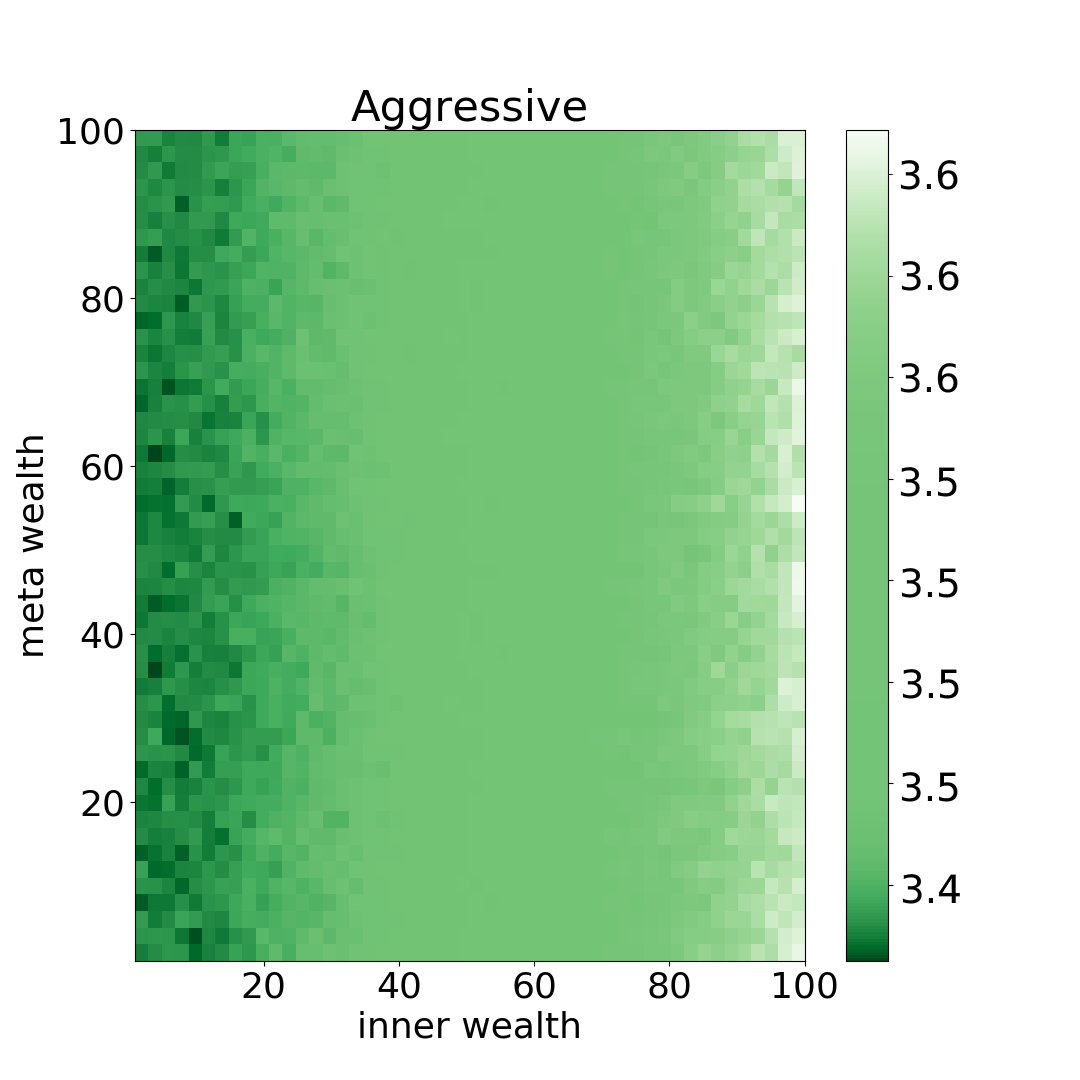}
%\vspace{-.3cm}
\end{minipage}
\begin{minipage}[t]{0.49\textwidth}
\centering
\includegraphics[width=.9\textwidth]{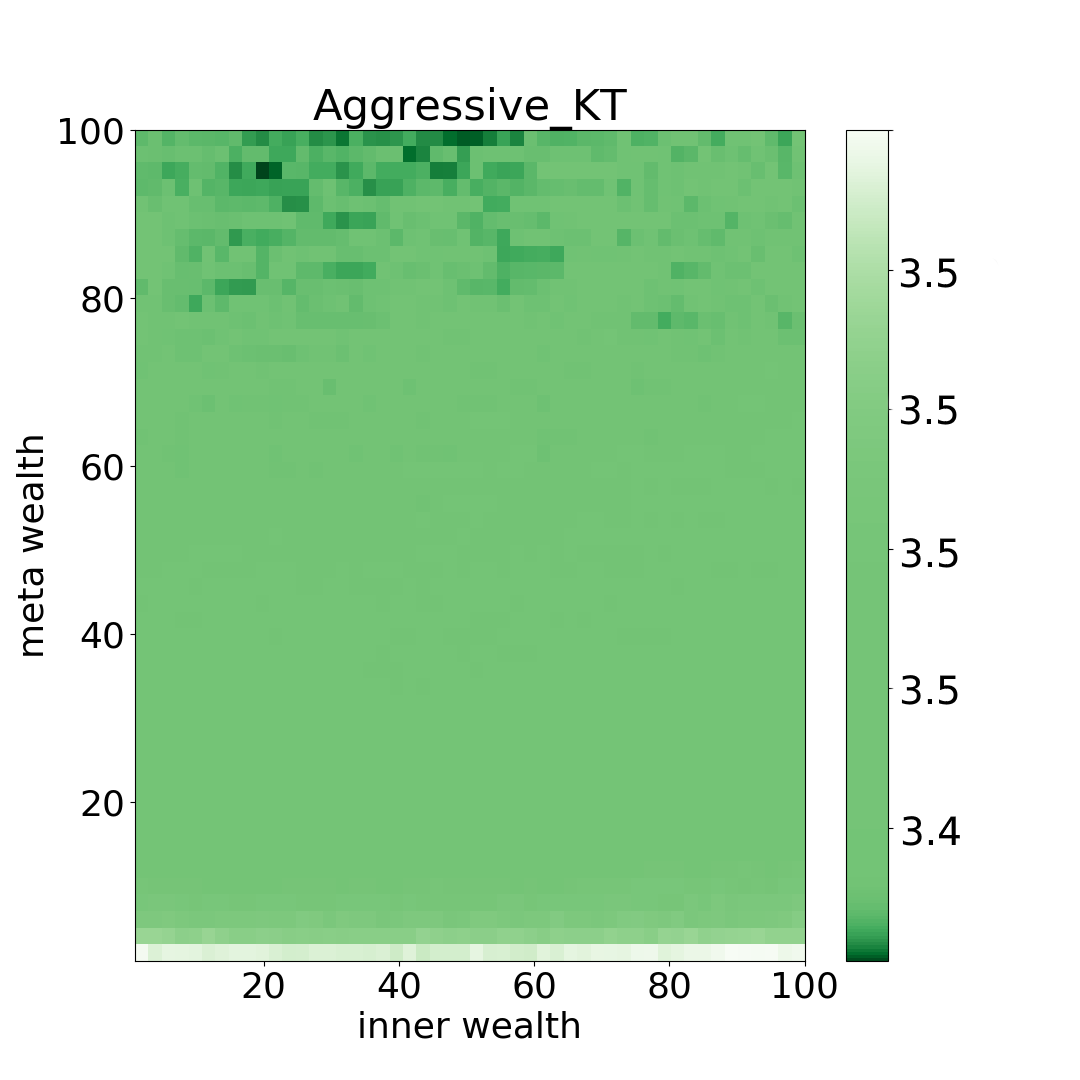}
%\vspace{-.3cm}
\end{minipage}
\caption{Average across-tasks cumulative error (over $30$ seeds) 
of our aggressive method w.r.t. a grid of inner-meta initial
wealths on synthetic data. Variant using refined coin betting (left), 
variant using KT coin betting (right). \label{fig_exps_grid}}
\end{figure}

In order to investigate the sensitivity of our parameter-free methods w.r.t. to the initialization of the wealths, we ran the aggressive variants of our method on the same experimental setting of \cref{fig_exps_real_synth} (top) over a $50 \times 50$ 
linearly spaced grid of (inner and outer/meta) initial wealths in the interval $[0.1, 100]$. 
In \cref{fig_exps_grid} we report the average across-tasks cumulative 
error we got at the end of the entire sequence of tasks for any value in 
the grid. From our results, we can observe that the performance of the method is  
quite stable and not much sensitive w.r.t. the initialization of the wealths. Hence, 
coherently to the single-task setting in \cite{orabona2016coin}, also in our multiple
tasks methods, the choice of the initial wealths has a mild impact on the performance.

%\gd{We also investigated the gap between the convergence rate of the aggressive version of our method and its lazy counterpart, as a function of the within-task points. In \cref{fig:synth-data-class} we report the average cumulative error of the two methods with different number of the within-task points $n$. As we can see, the plot is in line with our theory: the larger is $n$, the largerthe gap is between the two variant.}

%-------------------------------------------------------------------------------------------------------

\subsection{REAL EXPERIMENTS}

\begin{figure}[t]
\begin{minipage}[t]{0.49\textwidth}  
\centering
\includegraphics[width=1\textwidth]{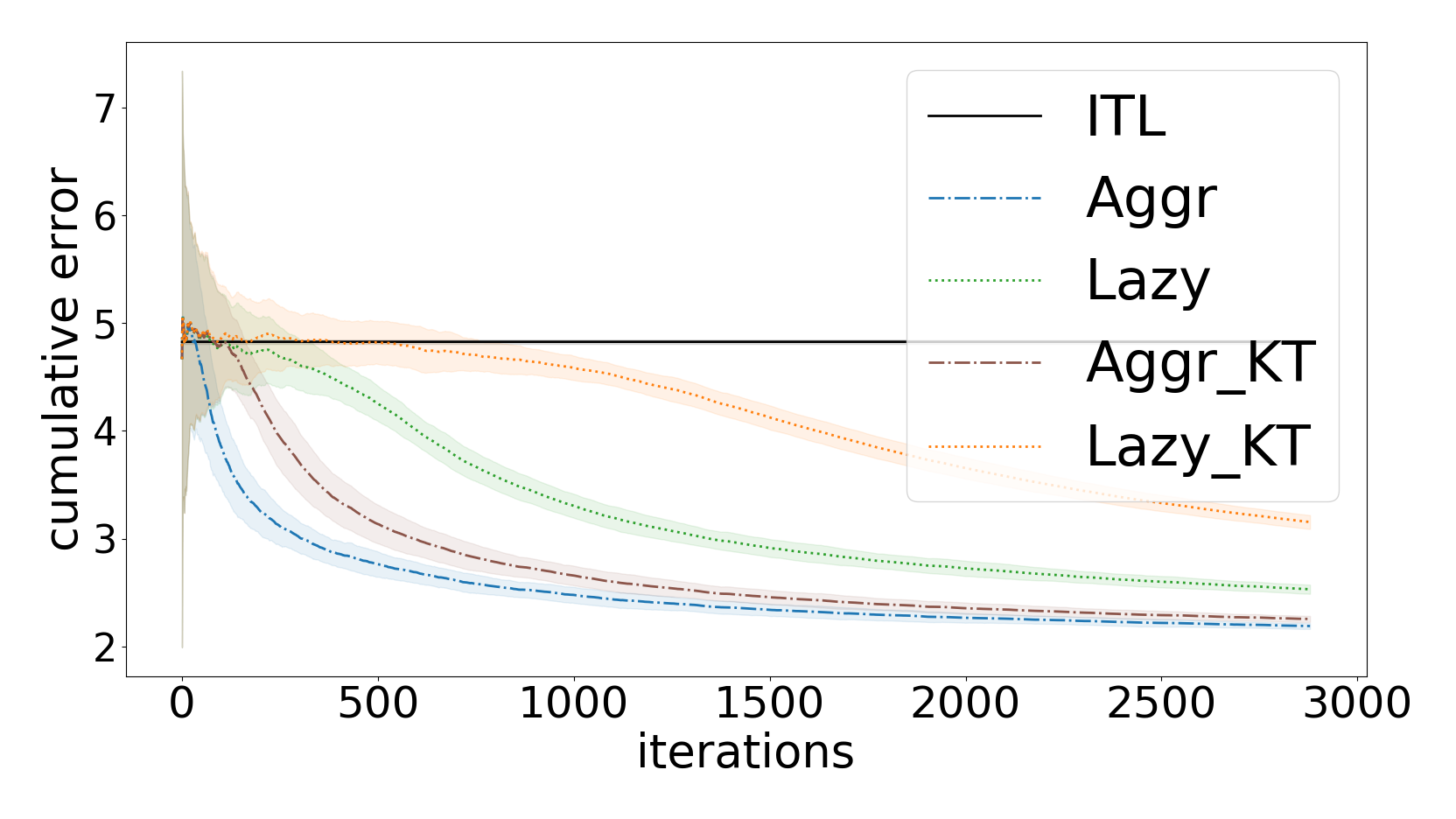}
%\vspace{-.3cm}
\end{minipage}
\begin{minipage}[t]{0.49\textwidth}
\centering
\includegraphics[width=1\textwidth]{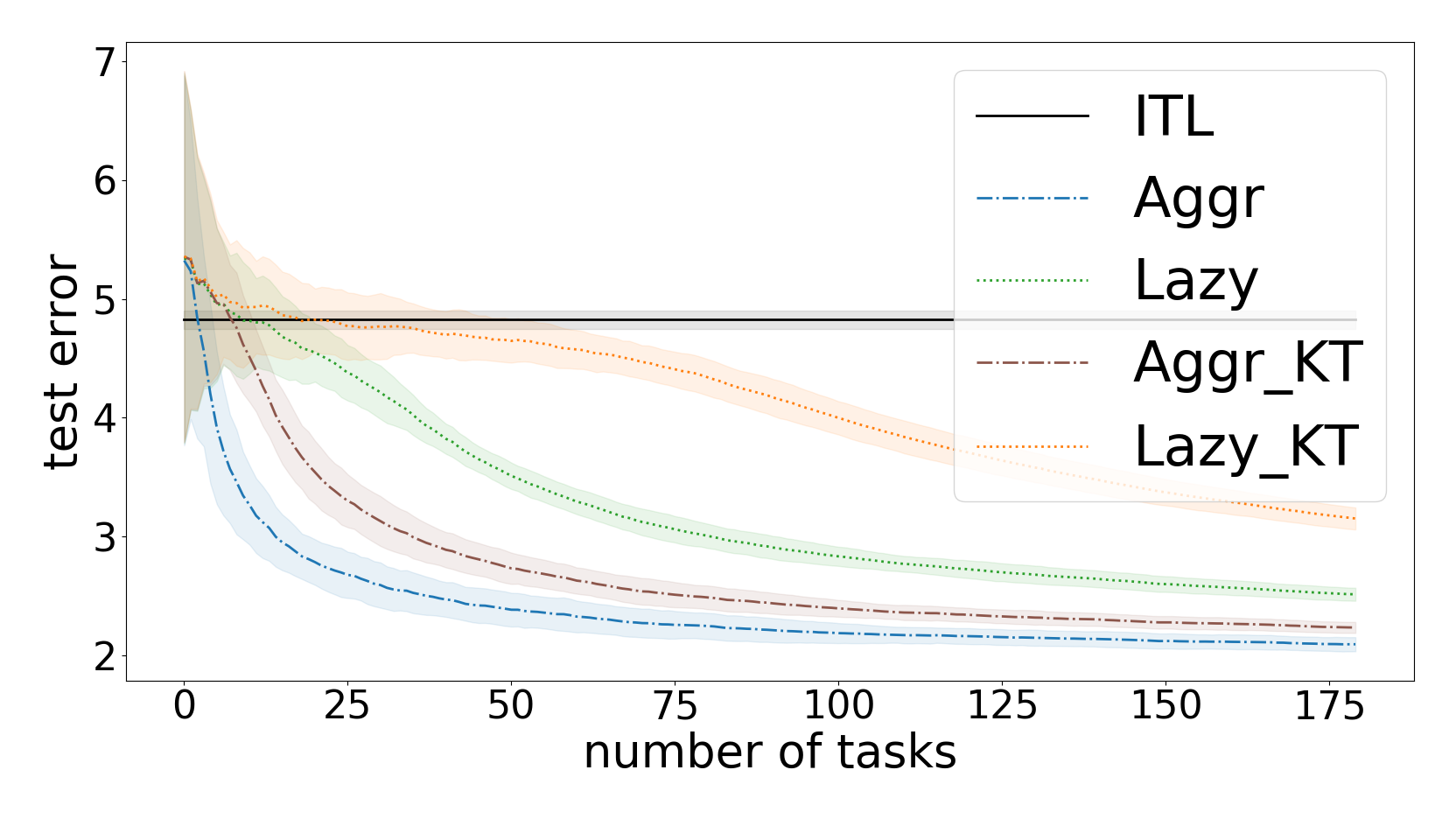}
%\vspace{-.3cm}
\end{minipage}
\caption{Average performance (over $30$ seeds) 
of different methods w.r.t. an increasing number of 
iterations or tasks on the Lenk dataset. Average 
across-tasks cumulative error (left), average multi-task 
test error (right). \label{fig_exps_real_lenk}}
\end{figure}

\begin{figure}[t]
\begin{minipage}[t]{0.49\textwidth}  
\centering
\includegraphics[width=1\textwidth]{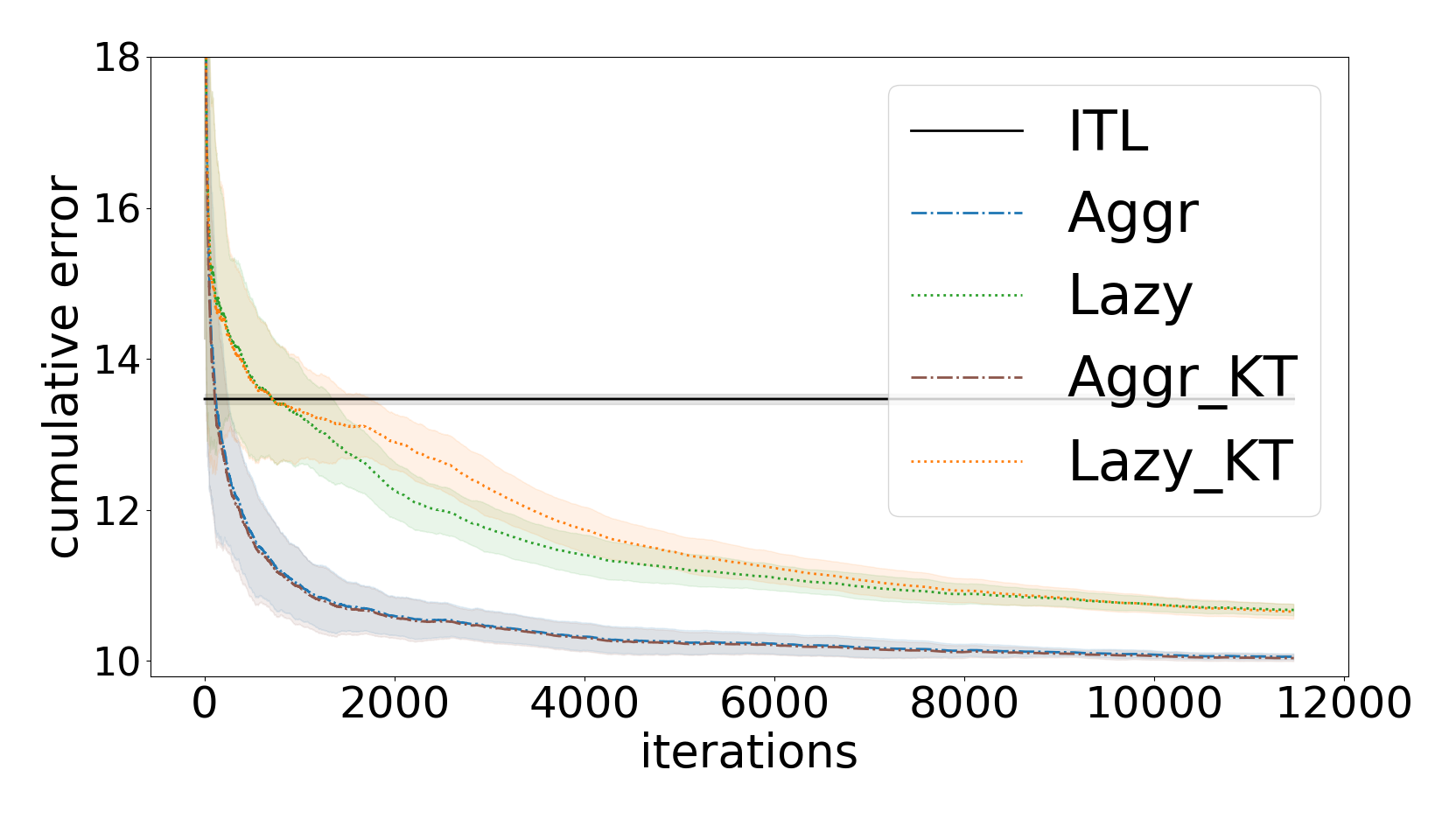}
%\vspace{-.3cm}
\end{minipage}
\begin{minipage}[t]{0.49\textwidth}
\centering
\includegraphics[width=1\textwidth]{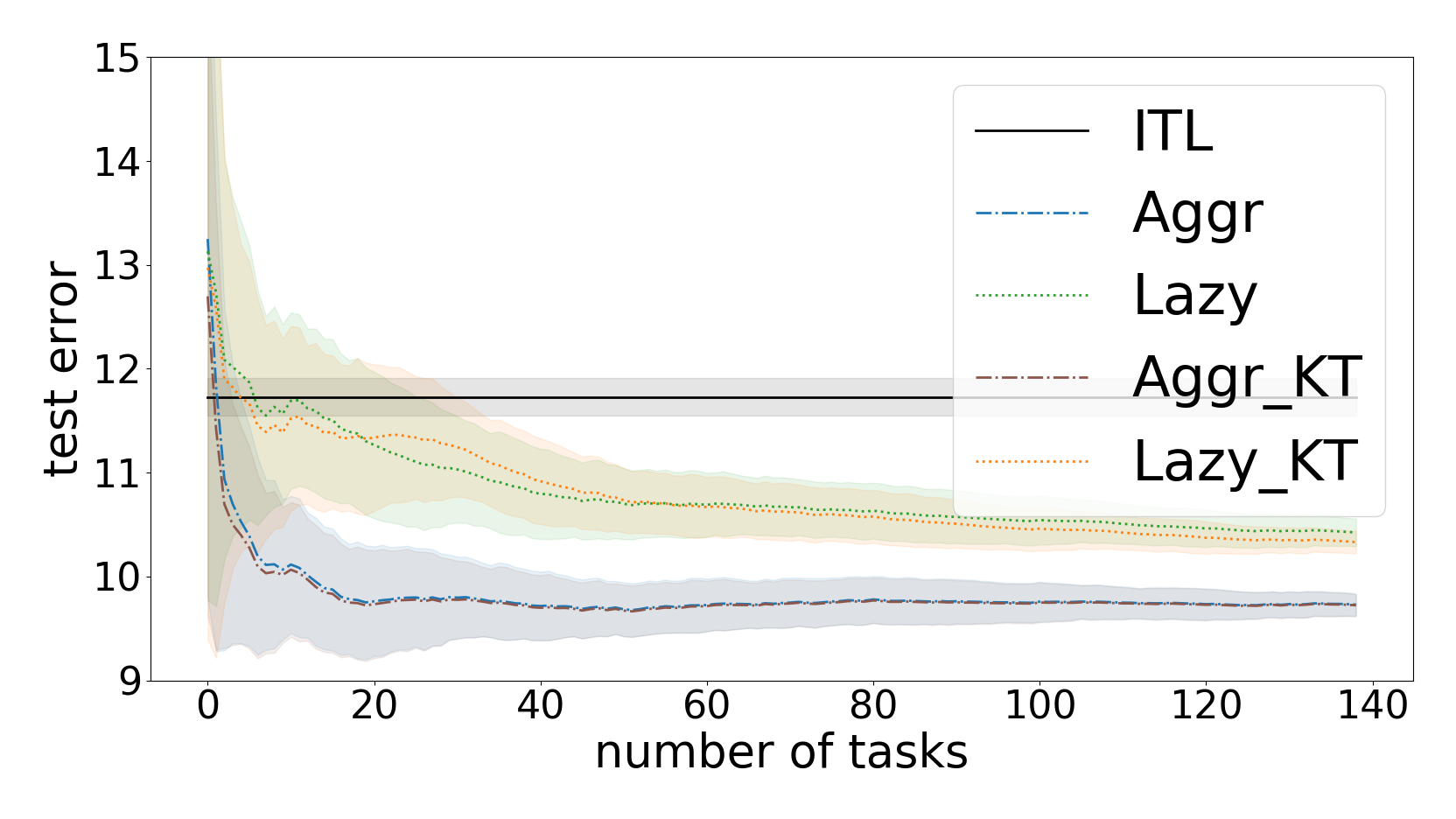}
%\vspace{-.3cm}
\end{minipage}
\caption{Average performance (over $30$ seeds) 
of different methods w.r.t. an increasing number of 
iterations or tasks on the Schools dataset. Average 
across-tasks cumulative error (left), average multi-task 
test error (right). \label{fig_exps_real_schools}}
\end{figure}

We tested the performance of our methods also on two regression 
problems on the Lenk and the Schools datasets. Also in these cases, 
we set the initial wealths in our methods equal to $1$, for both the 
inner and the outer algorithm. In the plots below, we used $80\%$ 
of the available datapoints for each task to train the inner algorithm.
The remaining part was used to compute the test error of the inner 
algorithm in the statistical multi-task setting.

{\bf Lenk dataset.} We considered the computer survey data from \cite{lenk1996hierarchical,Andrew}, in which $T = 180$ people 
(tasks) rated the likelihood of purchasing one of $n = 20$ 
different personal computers. The input represents $d = 13$ different 
computers' characteristics, while the output is an integer rating between 
$0$ and $10$. In \cref{fig_exps_real_lenk} we report the average 
across-tasks cumulative error (left) and the average multi-task test error (right) 
for all the methods w.r.t. to an increasing number of datapoints/iterations
or tasks. The results we got are in agreement with the synthetic experiments 
in the main body. Our parameter-free approaches significantly outperform 
ITL and they converge to the oracle (the algorithm with the best bias 
in hindsight) as the number of the observed datapoints/tasks increases. 
Again, the aggressive variants of our method present faster rates w.r.t. 
the corresponding lazy counterparts. We also observe that, in this setting, 
the KT variants of our parameter-free methods present a slightly slower 
convergence w.r.t. the corresponding refined variants.

{\bf Schools dataset.}
We considered the Schools dataset \cite{argyriou2008convex}, consisting of 
examination records from $T = 139$ schools. Each school is 
associated to a task, individual students are represented by a 
features' vectors $x \in \Real^d$, with $d = 26$, and their exam scores 
to the outputs. The sample size $n$ varies across the 
tasks from a minimum $24$ to a maximum $251$. The results we
got in \cref{fig_exps_real_schools} are coherent to the ones we 
described above for the Lenk dataset and they confirm the effectiveness
of our method also on this dataset. In this case, we observe that, the 
convergence speed of the KT variants of our parameter-free methods 
is equivalent to the one of the corresponding refined variants.

%We conclude this section reporting the characteristics of the machine 
%we used for running our experiments and the complexity of our method 
%in \cref{OGDA2_paper}.

%\gd{All the experiments were conducted on an Intel Xeon E5-2697 V3 2.60Ghz CPU with 32GB RAM.}
%\gd{All the experiments were conducted on a workstation with 4 Intel Xeon E5-2697 V3 2.60Ghz CPUs and 256GB RAM.}

%The variant of our method in \cref{OGDA2_paper} for biased regularization using the batch inner algorithm in \cref{RERM_bias} has a time and space complexity $\mathcal{O}(d (k+n))$. The variant for fine tuning using the online inner algorithm in \cref{online_inner_algorithm} has a time and space complexity $\mathcal{O}(dk)$.

\end{document}